\documentclass{article}

 \PassOptionsToPackage{numbers, compress}{natbib}
 \usepackage[preprint]{neurips_2020}
\usepackage{wrapfig}
\usepackage{caption}
\usepackage{microtype}
\usepackage{graphicx}
\usepackage{subcaption}
\usepackage[utf8]{inputenc} 
\usepackage[T1]{fontenc}    
\usepackage{hyperref}       
\usepackage{url}            
\usepackage{booktabs}       
\usepackage{amsfonts}       
\usepackage{nicefrac}       
\usepackage{microtype}      
\usepackage{algorithmicx,algpseudocode,algorithm}
\usepackage{color-edits}
\usepackage{enumitem}
\usepackage{amsmath,amsthm,amssymb}
\usepackage{cleveref}

\newtheorem{theorem}{Theorem}
\newtheorem{lemma}[]{Lemma}
\newtheorem{proposition}[]{Proposition}
\newtheorem{remark}[]{Remark}
\newtheorem{corollary}[]{Corollary}
\newtheorem{definition}{Definition}

\newtheorem{assumption}{Assumption}

\DeclareMathOperator{\vol}{vol}
\newcommand{\ps}{p_{\sigma^2}}
\newcommand{\gs}{g_{\sigma^2}}
\newcommand{\inprod}[2]{\left\langle #1, #2 \right\rangle}

\DeclareMathOperator{\ric}{Ric}
\DeclareMathOperator{\ent}{Ent}
\DeclareMathOperator{\trace}{Tr}
\newcommand{\bfx}{\mathbf{x}}
\newcommand{\bfz}{\mathbf{z}}
\newcommand{\mcal}{\mathcal}
\newcommand{\bfs}{\mathbf{s}}

\newcommand{\bftheta}{{\boldsymbol{\theta}}}

\hypersetup{
	colorlinks=true,
	linkcolor=blue,
	filecolor=blue,      
	urlcolor=blue,
	citecolor=blue
}
\addauthor{sasha}{red}

\title{Fast Mixing of Multi-Scale Langevin Dynamics under the Manifold Hypothesis}

\author{
  Adam Block \\
  Department of Mathematics\\
  MIT
   \And
   Youssef Mroueh \\
  IBM Research\\
  MIT-IBM Watson AI Lab
   \AND
  Alexander Rakhlin \\
   MIT \\
   MIT-IBM Watson AI Lab
   \And
   Jerret Ross \\
  IBM Research\\
  MIT-IBM Watson AI Lab
}

\begin{document}

\maketitle

\begin{abstract}
    Recently, the task of image generation has attracted much attention. In particular, the recent empirical successes of the Markov Chain Monte Carlo (MCMC) technique of Langevin Dynamics have prompted a number of theoretical advances; despite this, several outstanding problems remain. First, the Langevin Dynamics is run in very high dimension on a nonconvex landscape; in the worst case, due to the NP-hardness of nonconvex optimization, it is thought that Langevin Dynamics mixes only in time exponential in the dimension. In this work, we demonstrate how the manifold hypothesis allows for the considerable reduction of mixing time, from exponential in the ambient dimension to depending only on the (much smaller) intrinsic dimension of the data. Second, the high dimension of the sampling space significantly hurts the performance of Langevin Dynamics; we leverage a multi-scale approach to help ameliorate this issue and observe that this multi-resolution algorithm allows for a trade-off between image quality and computational expense in generation. 
\end{abstract}

\section{Introduction}
    Generative modeling has recently inspired great interest, especially with the empirical successes in everything from text generation to protein modeling \cite{Ingraham2018,Song2017}.  Of particular interest is the task of image generation, where a number of approaches over the past several years have produced increasingly sharp computer-generated images.  While much of the state of the art work has been with Generative Adversarial Networks (GANs) \cite{Goodfellow2014}, the difficulty in training and in interpretability, the relative lack of theory, and the desire for a likelihood-based approach have led to great interest in Markov Chain Monte Carlo (MCMC) methods for image generation.
    
    The MCMC method has been particularly fruitful, with such work as \cite{SongErmon2019,Nguyen2017} generating high-quality, clear, and diverse images.  The empirical success of such models is curious, as it diverges sharply from what might be expected given the current theory.  The motivation for such methods is that the algorithm is sampling from a Markov Chain that is converging quickly to a stationary distribution close to that of the population; the rate of convergence, as discussed below, is governed by a parameter of the population distribution called the log-Sobolev constant.  While it is classical at this point that certain distributions in high dimension, such as log-concave measures \cite{Emery1985}, have small log-Sobolev constants, guaranteeing fast mixing, in the typical case  these constants tend to grow exponentially with the dimension of the space on which the measure is supported.  In the context of image generation, where each pixel corresponds to a different dimension and where there is no convexity to accelerate mixing, one would expect the mixing to be so slow as to be prohibitive; however, the empirical results demonstrate that fast mixing is achieved.  Other methods for reducing the dimensional dependence of mixing times include making assumptions on boundedness \cite{Bardet2018}, symmetry \cite{Risteski2020}, or that the distribution is a mixture of log-concave densities \cite{Ge2018}; like the convexity assumption above, there is no reason to suppose that image data fit any of these criteria.  Thus, we have an enigma: why can Langevin dynamics generate sharp and natural images?

    While there is a plethora of theoretical analyses of Langevin Dynamics, we hew most closely to the regime studied in \cite{DAESGLD}, where we must both learn the density of the stationary distribution and run the MCMC method for learning.  As noted in that work, this extra learning step and the consequent lack of the exact population distribution, creates a trade-off when running LD: if the MCMC is not run for enough steps then the Markov chain does not mix, but if the algorithm is run for too many steps then it starts to diverge from the target distribution due to the estimate being used instead of the population distribution.  Indeed, as observed in \cite{SongErmon2019} and in our own experiments, if the LD is run for too many steps, this divergence that appears in the theory indeed shows up in practice.  This extra layer of complexity makes it even more important for the mixing time to be small, as, in reality, the algorithm cannot be run indefinitely.  Thus the empirical results are clear: Langevin dynamics applied to image generation mixes quickly, despite the high dimension and lack of convexity.
    
     We propose the manifold hypothesis as an explanation.  It has long been thought that certain high-dimensional data with complicated structure can be learned because they lie on or near a much lower-dimensional manifold embedded into the ambient space \cite{Fefferman2016}, while the empirical success of GANs, which use relatively low-dimensional Gaussians to produce high resolution images \cite{Zhang2020,Berthelot2020} provides strong evidence for the applicability of this model to image generation.  In this paradigm, the relevant measure of complexity is the \emph{intrinsic} dimension of the data rather than any \emph{extrinsic} features.  This hypothesis is particularly well suited to highly structured data where the data has complicated associations between its coordinates; images provide an excellent example.  Using techniques from Riemannian geometry, developed in the supplement, we show that under the manifold hypothesis, the mixing time of the Langevin dynamics used for image generation depends only on this intrinsic dimension, which is much smaller than the apparently very high dimension of the pixel space.
     
     The structure of the paper is as follows: in \Cref{sec2} we set up the mathematical framework required to state the results; in \Cref{sec3} we state the main results regarding the fast mixing time given the manifold hypothesis and quite general geometric assumptions; in \Cref{mrl} we suggest a modified algorithm that allows for a clearer understanding of the trade-off between runtime and image quality; finally, in \Cref{experiments} we use the algorithm suggested in \Cref{mrl} to provide compelling experimental evidence in favor of the  manifold hypothesis. All proofs are deferred to the supplement.
     
\section{Setup and Preliminaries}\label{sec2}
    \begin{wrapfigure}[14]{R}{0.55\textwidth}
\vspace{-2.6em}
\begin{minipage}{0.55\textwidth}
\begin{algorithm}[H]
	\caption{Annealed Langevin Dynamics \cite{SongErmon2019}}
	\label{alg:annealLD}
	\begin{algorithmic}[1]
	    \Require{$\{(\sigma_i, T_i)\}_{i=1}^{L}, \epsilon,  \bfs_\bftheta(\tilde{\bfx}, \sigma_i)$}\\ ($\bfs_\bftheta$ are score estimators)
	   \State{Initialize $\tilde{\bfx}_0 \sim \mathcal{N}(0,I)$}
	    \For{$i \gets 1$ to $L$}
	        \State{$\alpha_i \gets \epsilon \cdot \sigma_i^2/\sigma_{L^2}$} 
            \For{$t \gets 1$ to $T_i$}
                \State{Draw $\bfz_t \sim \mcal{N}(0, I)$}
                \State $\tilde{\bfx}_{t} \gets \tilde{\bfx}_{t-1} + \dfrac{\alpha_i}{2} \bfs_\bftheta(\tilde{\bfx}_{t-1}, \sigma_i) + \sqrt{\alpha_i}~ \bfz_t$
            \EndFor
        \EndFor\\
        \Return{$\tilde{\bfx}_{T_L}$}
	\end{algorithmic}
\end{algorithm}
\end{minipage}
\end{wrapfigure}
    We recall the necessary definitions and results that will be used in the sequel.  We let $\mu$ be a fixed reference measure and let $p = p d \mu$ be absolutely continuous to $\mu$.  If $p$ is differentiable, we define the \textbf{score} of $p$ as $\nabla \log p$.  In reality, we do not have access to the score of the population and instead must estimate it.  The technique used in \cite{SongErmon2019,DAESGLD} and applied in the sequel is one of De-Noising Score Matching (\cite{Vincent2011}).  Let $g_{\sigma^2}$ denote the density of a centred Gaussian in $\mathbb{R}^d$ covariance matrix $\sigma^2 I$.  We denote by $p_{\sigma^2} = p \ast g_{\sigma^2}$ the convolution of $p$ and $g_{\sigma^2}$.  The population and empirical denoising score matching (DSM) losses are defined respectively  for $X_i\sim p$ and $\xi_i\sim \mathcal{N}(0,I_d)$:
        \begin{equation*}
            \mathcal{L}_{DSM}(s) = \mathbb{E}_{Y \sim \ps} \left[||s(Y) - \nabla \log \ps(Y)||^2  \right] \quad \widehat{L}_{DSM}(s) = \frac 1n \sum_{i = 1}^n \left|\left| s(X_i+\sigma \xi_i ) + \frac{1}{\sigma}\xi_i \right|\right|^2
        \end{equation*}
    The empirical loss, minimized at $\mathbf{s}_\theta$, is known to consistently estimate the population loss \cite{Vincent2011,hyvarinen2005}.  Note that the smoothing properties of the Gaussian imply that $\ps$ always has a score with respect to the Lebesgue measure on $\mathbb{R}^d$, the first of many key advantages of the De-Noising approach.
    
    The foundation of the sampling scheme, proposed in \cite{SongErmon2019} and analyzed in \cite{DAESGLD} is annealed Langevin sampling, detailed in \Cref{alg:annealLD}.

    We interpret the Langevin algorithm as a discrete approximation of a continuous diffusion.  In order to establish notation, in the sequel, we denote by $X_t, \widehat{X}_t$  the diffusions started at some $X_0 \in \mathbb{R}^d$ governed by
    \begin{align}\label{eq1}
        d X_t = \nabla \log p_{\sigma^2}(X_t) d t + \sqrt{2} d B_t && d \widehat{X}_t = \mathbf{s}_\theta(\widehat{X}_t) d t + \sqrt{2} d B_t
    \end{align}
    where $B_t$ is a standard $d$-dimensional Brownian motion and $\mathbf{s}_\theta$ is an estimate of $\nabla \log \ps$.  We denote by $\nu_t$ the law of $X_t$ and $\widehat{\nu}_t$ the law of $\widehat{X}_t$.  Note that, under quite general conditions, $\nu_t$ converges to $\ps$ (\cite{Bakry2014}).  In order to quantify such convergence, we consider Wasserstein distance.
    \begin{definition}
        Let $\mu, \nu$ be two distributions on $\mathbb{R}^d$ with finite second moments.  We define the Wasserstein 2-distance as
        \begin{equation}
            \mathcal{W}_2(\mu, \nu)^2 = \inf_\Gamma \mathbb{E}_{(X, Y) \sim \Gamma}\left[||X - Y||^2\right]
        \end{equation}
        where the infimum is taken over all laws $\Gamma$ with marginals $\mu$ and $\nu$.
    \end{definition}
    While we defer a rigorous definition of a log-Sobolev inequality to the supplement, we note that the classical theory of diffusions tells us that if $\ps$ satisfies a log-Sobolev inequality with constant $c_{LS}(\ps)$, then
    \begin{equation}
        \mathcal{W}_t\left(\nu_t, \ps\right) \leq \mathcal{W}_2\left(\nu_0, \ps\right) e^{- \frac{2t}{c_{LS}(\ps)}}
    \end{equation}
    For the sequel, in order to apply the results of \cite{DAESGLD}, we need one last definition:
    \begin{definition}
        Let $f: \mathbb{R}^d \to \mathbb{R}^d$ be a vector field.  We say that $f$ is $L$-Lipschitz if for all $x, y \in \mathbb{R}^d$, we have
        \begin{equation}
            ||f(x) - f(y)|| \leq L ||x - y||
        \end{equation}
        We say that $f$ is $(m,b)$-dissipative if for all $x \in \mathbb{R}^d$, we have
        \begin{equation}
            \langle f(x), x \rangle \geq m ||x||^2 - b
        \end{equation}
    \end{definition}
    In order to motivate this definition, we note that if the score of $\ps$ is $(m,b)$-dissipative then it satisfies a log-Sobolev inequality (\cite{RakhlinRaginsky}).
    
    Clearly, as we only have access to (a discrete approximation of) $\widehat{X}_t$ and not $X_t$, we hope that $\mathcal{W}_2(\widehat{\nu}_t, p)$ is small; indeed, we have the following result from \cite{DAESGLD}:
    \begin{theorem}[Theorem 12 from \cite{DAESGLD}]\label{lastthm}
        Let $d \geq 3$ and suppose that the scores of $p$ and $\ps$ are $L$-Lipschitz and $(m,b)$-dissipative.  Let $\mathbf{s}_\theta$ be an estimate of the score of $\ps$ whose expected squared error with respect to $\ps$ is bounded by $\varepsilon^2$.  Suppose that $X_0 \sim \widehat{\nu}_0$.  Under technical conditions on $\widehat{\nu}_0$ satisfied by a multivariate Gaussian, we have
        \begin{equation}\label{eq2}
            \mathcal{W}_2\left( \widehat{\nu}_t, p\right) \leq \sigma \sqrt{d} + \mathcal{W}_2\left(\widehat{\nu}_0, \ps \right) e^{- \frac{2 t}{c_{LS}(\ps)}} + C \sqrt{(b + d) t} \left(\varepsilon t + \left|\left|\ps\right|\right|_\infty^{\frac 12 - \frac 1d} e^{\frac{L \sqrt{d}}{4} t} \sqrt{t} \varepsilon^{\frac 1d}\right)^{\frac 14}
        \end{equation}
        where $C$ does not depend on the dimension.
    \end{theorem}
    In the remainder of this paper, for the sake of simplicity, we neglect the effect that the discretization error has on the Wasserstein distance, i.e. the fact that our algorithm is only sampling from a discretized approximation of a continuous diffusion.  While work such as \cite{DAESGLD,Li2019} examines this aspect, it lies more in the realm of numerical analysis than statistical theory and, moreover, considerably reduces the clarity of the results without adding significant insight. 
    
\section{The Manifold Hypothesis}\label{sec3}
    
    The starting point for our analysis is \Cref{lastthm}.  As detailed in \cite{DAESGLD}, the bound of \Cref{eq2} can be broken into three sources of error: $\mathcal{W}_2(\ps, p)$, the error due to noising the data; $\mathcal{W}_2(\nu_t, \ps)$, the error due to the Markov chain not having fully mixed; $\mathcal{W}_2(\nu_t, \widehat{\nu}_t)$, the error due to running the algorithm with a score estimator rather than the actual score.  While the first term is unaffected by the evolution time of the diffusion, the last two terms act against each other; thus, when running Langevin dynamics with only an approximation of the score, the bounds suggest that the   algorithm can only be run for a finite amount of time before the diffusion begins to diverge.  This exact effect appears in the experiments using Langevin sampling to generate images and is precisely what motivates the annealing introduced in \cite{SongErmon2019}.  Thus, because the diffusion cannot be run indefinitely, even without regard to practical considerations, for it to be successful we must have very fast mixing as measured by the log-Sobolev constant.

    Without any further assumptions, due to the NP-hardness of nonconvex optimization, the log-Sobolev constant is typically exponential in the dimension of the space; nonetheless, the above considerations suggest that we are not in the worst case.  The most classical way to remove this curse of dimensionality is with an assumption of convexity, using the Bakry-Emery criterion \cite{Emery1985}.  Unfortunately, there is no reason to believe that in most of the interesting cases, such as image generation, the convexity assumption holds.  Our main result is, under the manifold hypothesis, the log-Sobolev constant is independent of the extrinsic dimensionality and depends only on the intrinsic geometry of the underlying data manifold; thus, if we assume that the intrinsic dimension $d'$ is much smaller than the apparent dimension $d$, the theory predicts the fast mixing that is observed empirically.  
    
    While we focus on the more theoretically neat case of the population distribution lying on the low-dimensional manifold, perhaps a more realistic assumption is that the set of feasible images is close in some sense to this manifold.  There are various ways to extend our results to this setting.  For example, if we assume that the images are obtained from the low-dimensional structure by adding small amounts of log-concave noise, then \Cref{lem1} and the Bakry-Emery criterion ensure that the fast mixing properties still hold.
    
    For rigorous statements, the language of Riemannian geometry is particularly useful; a quick review with references is provided in the supplement.  We make the following assumption on the population distribution $p$:
    \begin{assumption}\label{as1}
        Let $(M,g)$ be a $d'$-dimensional, smooth, closed, complete, connected Riemannian manifold isometrically embedded in $\mathbb{R}^d$ and contained in a ball of radius $\rho$, such that there exists a $K \geq 0$ such that $\ric_M \succeq -K g$ for all $y \in M$ in the sense of quadratic forms.  With respect to the inherited metric, $M$ has a volume form $\vol$, which has finite total integral on $M$ due to compactness.  Then $p = p \vol_M$ is absolutely continuous with respect to the volume form and we refer to its density with respect to this volume form as $p$ as well, by abuse of notation. 
    \end{assumption}
    As an aside, we note that if $K < 0$ and we had control over the Hessian of $p$ then we would be back in the convex case covered by the Bakry-Emery criterion.  As we do not expect this convexity to hold in the relevant application of image generation, we restrict our focus to the nonconvex case.  Throughout, we consider the regime where $K$ is "large" in the sense that we consider $K > c$ for a fixed constant $c > 0$ and the order constants appearing in the results below depend on this $c$.  This greatly simplifies the appearance of the bounds without having any adverse affect on generality, as the strict dependence can be traced in the proofs. 

    First, we establish that we are, indeed, in the setting of \Cref{lastthm}:
    \begin{proposition}\label{prop1}
        Suppose that $p$ is as in \Cref{as1}.  Then $\nabla \log \ps$ is $\frac{\rho^2}{\sigma^4}$-Lipschitz if $\rho^2 \geq \sigma^2$ and $\frac 1{\sigma^2}$-Lipschitz otherwise.  In either case, $-\nabla \log \ps$ is $\left(\frac{1}{2\sigma^2}, \frac{\rho^2}{2 \sigma^2} \right)$-dissipative.
    \end{proposition}
    Second, we note that the Gaussian fattening that is needed for the score estimation step (see \cite{DAESGLD,SongErmon2019}) does not adversely affect the mixing time by "forgetting" the low dimensional structure.  Indeed:
    \begin{lemma}\label{lem1}
        If $p$ has log-Sobolev constant $c$ then $c_{LS}(\ps) \leq 2 \sigma^2 + c$.
    \end{lemma}
    As mentioned above, it is well-known that convexity leads to fast mixing.  We introduce a measure of nonconvexity related to the Kato constant from Riemannian geometry (see \cite{Rose2017,Rose2019} for example).
    
    \begin{definition}\label{def1}
        Suppose we are in the situation of \Cref{as1}.  Let $\ric_-(x)$ denote the smallest eigenvalue of $\ric$ at $x$.  For any $R > 0$, we define
        \begin{equation}
            \kappa(R) = \sup_{x \in M} \frac{1}{\vol(B_R(x))} \int_{B_R(x)} (d' - 1 - \ric_-)_+ d \vol_M
        \end{equation}
        where $B_R(x)$ is the metric ball of radius $R$ centred at $x$ in $M$.  We let $\kappa = \kappa\left(\sqrt{\frac{d'-1}{K}} \log 2\right)$.
    \end{definition}
    
        The function $\kappa(R)$ measures the failure of $M$ to be locally convex at all scales.  If $M$ is positively curved at all points, then $\kappa(R) \leq d'-1$ for all $R$.  If, however, there are neighborhoods of negative curvature, then $\kappa(R)$ can get larger and the amount by which $\kappa(R)$ can grow is determined by the magnitude of the negative curvature coupled with the size of the neighborhood of negative curvature.  In this way, we treat $\kappa(R)$ for fixed $R$ as a measure of how nonconvex $M$ is at the scale $R$.  The fixed radius $R = \sqrt{\frac{d'-1}{K}} \log 2$ is for the sake of concreteness, allowing nice bounds on the log-Sobolev constant in terms of the relevant parameters.
    
    With the above measure in hand, we are able to bound the log-Sobolev constant in terms of the purely geometric quantities of intrinsic dimension and curvature estimates. 
    \begin{theorem}\label{cls1}
        Suppose that $p$ satisfies \Cref{as1} and suppose that $\nabla \log p$ is $L$-Lipschitz and that $||\nabla \log p||_g \leq B$ at all points.  Let $\kappa = \kappa\left(\sqrt{\frac{d'-1}{K}} \log 2 \right)$  Then, if $K > \frac 1{d'}$,
        \begin{equation}
            c_{LS}(p_{\sigma^2}) = \widetilde{O}\left(\sigma^2 + d'^2 K \log \kappa e^{L B^2 d'^2 \log^2 \kappa} \right)
        \end{equation}
        where $\widetilde{O}$ indicates that we are ignoring factors logarithmic in $d'$ and $K$.  The explicit dependence on all constants can be found in the supplement.
    \end{theorem}
    A few remarks are in order.  First, note that in the presence of convexity, the bound is loose, in that the Bakry-Emery criterion gives a dimension-independent rate.  Nevertheless, the dependence of the bound on the two measures of nonconvexity that we have available: the maximum magnitude of the negative curvature and the average amount of negative curvature in the manifold.  We remark that by the definition of $\kappa$, we always have that $\kappa \leq K + d'$.  Moreover, as can be seen from the proof in the supplement, the constant $\kappa$ controls the diameter of the manifold; if we had finer control over this quantity, the bound can be much smaller.
    
    The bound in \Cref{cls1} should be compared to that appearing in \cite{Bardet2018}, where there is no assumption of low-dimensionality, but rather boundedness.  In our case, if we assume that $\sigma \leq \rho$, the maximum norm of a point in $M$, Theorem 1.3 from \cite{Bardet2018} implies that
    \begin{equation}
        c_{LS}(\ps) = O\left(\left(d + \frac{\rho^2}{\sigma^2} \right)\rho^2 e^{\frac{\rho^2}{\sigma^2}} \right)
    \end{equation}
    While this bound has reasonable dimensional dependence, without control of $\rho$ and with the small noise levels used in the annealed Langevin Dynamics, the constant becomes prohibitively large.
    
    With stronger assumptions on $p$, we can get tighter bounds.  In particular, if we suppose that $p$ is uniform on the manifold $M$, then a similar technique implies:
    \begin{theorem}\label{cls2}
        Suppose that the pair $(M, g)$ satisfies \Cref{as1} and let $p \propto \vol_M$ be uniform on $M$.  Assume that $K > 1$ and that $\kappa > 1$.  Then
        \begin{equation}
            c_{LS}(p_{\sigma^2}) = O\left(\sigma^2 + K^4 d'^2 \kappa^{20 K^2 d'}\right)
        \end{equation}
    \end{theorem}
    \vskip -0.1in
    Arguably, for the purpose of image generation, assuming a uniform distribution on a manifold is reasonable: there is no structural reason to privilege one image over another, thus we need only desire a generated point to be on the manifold, without any greater specificity.
    
    Crucially, the above bound is completely intrinsic to the geometry of the data manifold and that the dimension of the feature space does not appear.  This explains the fast mixing times of Langevin dynamics for image sampling: even with arbitrarily high dimension in pixel space, if the feasible space has small dimension $d'$, Langevin dynamics will still mix quickly.

\section{Multi-Resolution Annealed Langevin Sampling}\label{mrl}
    As motivation for the theoretical analysis above, we noted that the bound in \Cref{lastthm} has two terms that describe the trade-off between better mixing and worse divergence from the diffusion driven by the population score.  In this section, motivated by the low dimensional hypothesis and grounded in the theory in \cite{Burt1983,Mallat1989} as well as the recent empirical success of \cite{Denton2015,Karras2018progressive}, we propose an algorithm to take advantage of the posited structure.  As an added bonus, this algorithm provides a test of whether the exponential dependence in the last term of \Cref{eq2} is tight.  
    
    Because the Langevin dynamics is run in the pixel space, which is high dimensional, it is computationally expensive to run the algorithm at a reasonable resolution.  Thus, in the flavor of GANs, it would be much nicer to sample from a lower dimensional space and then transform this low dimensional representation into the high dimensional image.  While GANs have to learn a complicated map that pushes forward a simple distribution (like a Gaussian) to a complicated one, the use of the Langevin sampling at the lower resolution allows one to use a much simpler map that acts as an embedding, pushing forward a lower dimensional manifold into a higher dimensional space.  For this purpose, we use a simple upsampling operation that takes $m \times m$-pixel images and makes them $(2m) \times (2m)$-pixels.  The experiments section below details the precise operation.  For this method to have any hope of success, we would need to ensure that these push forward maps preserve the fast mixing properties proved in the preceding section.  Fortunately, we have:
        \begin{proposition}\label{prop5}
        Let $p$ satisfy a log-Sobolev inequality with constant $c_{LS}$.  Let $P$ be a $C^1$ projection with Jacobian $J_{P}$ that, at each point, has $\ell^2 \to \ell^2$ operator norm bounded by 1.  Then $p' = P_\# p$ satisfies a log-Sobolev inequality with the same constant.
    \end{proposition}
    \begin{remark}
        Note that the regularity assumption in \Cref{prop5} can be relaxed to weak differentiability coupled with a condition on the weak derivative being contractive when evaluated on the $\ell^2$ norm; this then includes neural networks with ReLU nonlinearity and 
        regularized weight matrices.
        \end{remark}
\begin{wrapfigure}[22]{r}{0.55\textwidth}
\vspace{-2em}
\begin{minipage}{0.55\textwidth}
\begin{algorithm}[H]
	\caption{Annealed Multi-Resolution  Langevin dynamics.}
	\label{alg:anneal}
	\begin{algorithmic}[1]
	    \Require{$\{\sigma_i\}_{i=1}^{L_j}, \epsilon, T_j,  \{ \bfs^j_\bftheta(\tilde{\bfx}, \sigma_i) \}_{j=0\dots J}$}\\ ($\bfs^j_\bftheta$ are score functions for each resolution (dimension $d_j$) )
	    \State $d_{j} = \frac{d}{2^{j}} , j =0\dots J$
	   \State{Initialize $\tilde{\bfx}_0 \sim \mathcal{N}(0,I_ { d_{j}} )$}
	    \For {$j \gets J$ to $0 $}
	    \For{$i \gets 1$ to $L_j$}
	        \State{$\alpha_i \gets \epsilon \cdot \sigma_i^2/\sigma_{L^2_j}$} 
            \For{$t \gets 1$ to $T_j$}
                \State{Draw $\bfz_t \sim \mcal{N}(0, I_{d_j})$}
                \State $\tilde{\bfx}_{t} \gets \tilde{\bfx}_{t-1} + \dfrac{\alpha_i}{2} \bfs_\bftheta^j(\tilde{\bfx}_{t-1}, \sigma_i) + \sqrt{\alpha_i}~ \bfz_t$
            \EndFor
        \EndFor
                    \State{$\tilde{\bfx}_0 \gets \tilde{\bfx}_T \uparrow 2$} (upsample )

        \EndFor
        \item[]
        \Return{$\tilde{\bfx}_{T_{L_0}}$}
	\end{algorithmic}
\end{algorithm}
\end{minipage}
\end{wrapfigure}
     \vskip -0.12 in

    With \Cref{prop5} in hand, we propose \Cref{alg:anneal}.  If it succeeds at all, it suggest that the low dimensional structure posited in the previous section at least approximately holds, as the low dimensional, early steps of the annealed Multi-Resolution Langevin are carrying enough information to be upsampled into a distribution close enough to the population so as to produce reasonably clear pictures.  On the other hand, if this algorithm does not substantially improve on the experimental results of \cite{SongErmon2019}, then it suggests that the exponential dimensional dependence appearing in the final term of \Cref{eq2} from \cite{DAESGLD} is not tight as, if it were, the lower dimensional sampling should allow the Langevin diffusion to be evolved further in time without negative effects.  We defer discussion of the results to the following section.
     
    While the focus of this work is not on score estimation, it should be noted that \Cref{alg:anneal} helps in this regard.  In order to train the score estimators used in \Cref{alg:anneal}, data is downsampled to each resolution and different score estimators are trained.  While some work on score estimation manages to avoid poor dimensional dependence by adding strong conditions on the target density, such as \cite{Gretton2018}, in general, without further assumptions, there can be very bad dimensional dependence in the score estimation, as noted in \cite{DAESGLD}; thus, the multi-resolution approach allows some of the hard work of score estimation to be transferred to the easier, lower-dimensional regime. In both the score estimation and the sampling steps, the advantages are not merely statistical, but computational as well.  Training and sampling in high dimension is  compute intensive and 
    moving parts of the expenses to cheaper regimes accelerates the runtime.

\section{Experiments}\label{experiments}
While Song \& Ermon experimented with image generation using annealed Langevin on images of size 32x32, in order to see if Langevin sampling would suffer from curse of dimension we focus in this Section on generating CelebA faces in 64x64.  

\textbf{Estimating the Score functions with DSM for Multiple Resolutions} To illustrate the manifold hypothesis in image generation and the effect of multi-scale Langevin, we  train a  Noise Conditional Score Network (NCSN) \cite{SongErmon2019} (i.e using Denoising Score Matching (DSM)) on CelebA dataset. In particular, we train two score networks for two resolutions: $32\times 32$  and  $64 \times 64$ images. For each resolution we consider $10$ level of noises and train under the DSM Objective as in \cite{SongErmon2019}. We note $s_{r}(x,\sigma)$, the score network at resolution $r$ for $r=32,64$.
$\sigma \in \{ 1.,  0.59,0.35, 0.21,0.12,0.07, 0.04,0.027,0.016, 0.01 \}$. 

\textbf{Model Selection} As in \cite{SongErmon2019} our score network is trained using SGD, with the default hyper-parameters  provided in \cite{SongErmon2019}. Similarly to \cite{SongErmon2019}, we do early stopping in order to select the best model, by monitoring the Frechet Inception Score (FID), as introduced in \cite{FID}, of $1000$ generated images via the annealed Langevin sampling, and select the model with the lowest FID. The model selection is performed on each resolution independently, and the annealed Langevin sampling starts from noise.

\textbf{Upsampling and Super-Resolution} For upsampling, we experimented with bi-cubic interpolation and with a pretrained Fast Super-Resolution Convolutional Neural Network (FSRCNN) \cite{dong2016accelerating}. We use the pretrained FSRCNN network provided in Pytorch at \cite{code} . 
These two methods were on par in terms of image quality. We adopt in the following FSRCNN for upsampling.

\textbf{Multi-resolution Langevin Sampling} We first reproduce the CelebA generation reported  in \cite{SongErmon2019} on on $32\times 32 $ images  and confirm as reported in the Appendix of  \cite{SongErmon2019} that model selection is crucial for the generated images quality.  In the following we compare three variants of Langevin sampling of CelebA $64\times 64$. 
We run Langevin sampling for $100$ iterations within each noise level for all variants.
\begin{enumerate}[align=left]
   \itemsep-0.05em 
    \item \emph{High Resolution Langevin Sampling (HRS):} Using the score network $s_{64}(x,\sigma)$ for $10$ levels of noise $\sigma$, we run the annealed Langevin sampling starting from random 64$\times$64 images.
    \item \emph{Low Resolution Langevin Sampling and Upsampling  (LRS-$\uparrow$)}: Using the score network $s_{32}(x,\sigma)$ for $10$ levels of noise $\sigma$, we run the annealed Langevin sampling to generate $32\times 32$ images starting from random images. We then input the output of the Langevin sampling to the upsampling function (FSRCNN) to obtain $64\times 64$ images .
    \item \emph{Low Resolution Sampling , upsampling and High Resolution sampling (LRS-$\uparrow$-HRS) } We generate $32\times 32$ images using the annealed Langevin sampling for $10$ noise levels. We upsample the images using FSRCNN to $64\times 64$. Then we initialize the annealed Langevin sampling using $s_{64}$ with the resulting images from upsampling. We  refer \emph{LRS-$\uparrow$-HRS-9} when the number of noise level at 64x64 resolution is $9$ starting from $\sigma=0.59$; \emph{LRS-2-$\uparrow$-HRS-9} stands for running low resolution Langevin for only the first two noise levels. \emph{LRS-$\uparrow$-HRS-3} when we only use the last three smallest noise levels in the HR sampling.
\end{enumerate}

\vskip -0.1in
\begin{table}[ht!]
\begin{center}
\begin{tabular}{lcc}
        \toprule
        Method & FID & Sampling Time \\
        & &(in seconds)\\
        \midrule
        \multicolumn{3}{l}{\textbf{CELEBA- $32\times 32$ }} \\
        \midrule
        LR Langevin  (LRS)  ~ & $48.53 $ & $170$\\
        \midrule
        \multicolumn{3}{l}{\textbf{CELEBA- $64\times 64$}}\\
        \midrule
        HR Langevin (HRS) ~ & $20.17$ & $556$ \\
        LR Langevin + Up (LRS-$\uparrow$) & $37.20$ &  $\mathbf{180}$\\
        mr-Langevin (LRS-$\uparrow$-HRS-3)  &$32.46$ & $348$\\
        mr-Langevin (LRS-$\uparrow$-HRS-9)&${26.44}$ &$680$\\
        mr-Langevin (LRS-2-$\uparrow$-HRS-9)&$\mathbf{19.54}$ &$533$\\
        WGAN-GP (DCGAN) \cite{TTUR}   & $21.4$& -\\
        WGAN-GP (DCGAN+TTUR) \cite{TTUR} &$\mathbf{12.5}$ & -  \\
        \bottomrule
        
    \end{tabular} 
\end{center}
\caption{ Image Quality and Sampling Time Tradeoffs: FID scores for CELEBA and Time in seconds for generation of 100 images using 100 iterations within each noise level of Langevin sampling.} \label{tab:score}
\vskip -0.23in
\end{table}
\textbf{Manifold Hypothesis, Image Quality and Computation Tradeoffs.} We evaluate the FID \cite{heusel2017gans} of the generated images using the sampling schemes described above. FID scoring follows the usual protocol and is evaluated on 10K samples from Langevin. Results are summarized in Table~\ref{tab:score}, we see that Langevin sampling is on par in terms of FID with a variant of WGAN-GP \cite{TTUR} and hence does not suffer from the curse of dimensionality, giving evidence to the manifold hypothesis.
Trajectories and samples of HRS can be seen in Figure~\ref{fig:scratch}.
Running Langevin sampling in high dimension is computationally expensive, taking 556 seconds for a batch of 100 images.  Using the multi-resolution scheme such LSR-$\uparrow$-3 results in a reduction in the sampling time at the price of a decrease in the image quality when compared with HR Langevin (See Figure \ref{fig:3traj}). The multiresolution scheme LRS-2-$\uparrow$-HRS-9, results in higher quality images than the high resolution langevin (FID 19.54 versus 20.17, See Fig \ref{fig:my_label9} in Appendix).   
Additional results can be found in Part E of the supplement.  
\begin{figure}
\begin{subfigure}[t]{0.5\textwidth}
\centering
\includegraphics[width=0.70\textwidth]{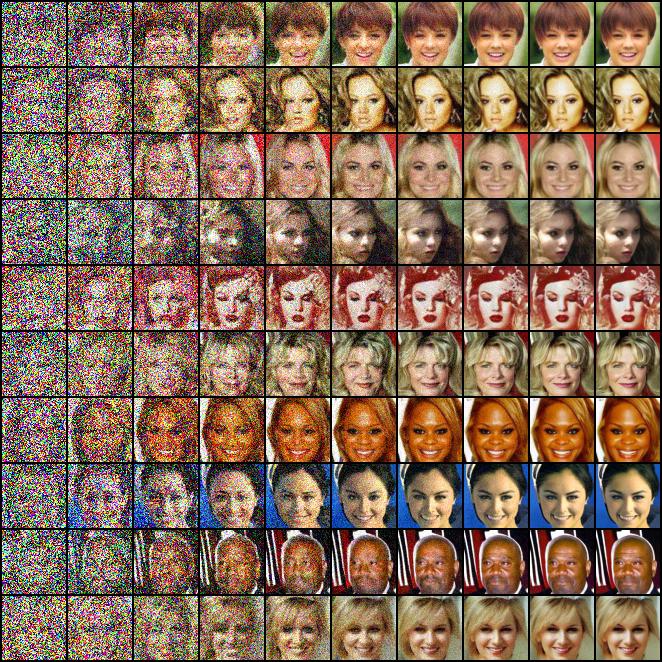}
        \caption{Trajectories of HR Annealed Langevin\\ on CelebA $64\times 64$  }
        \label{fig:scratchtraj}
    \end{subfigure}
 \begin{subfigure}[t]{0.5\textwidth}
        \centering 
        \includegraphics[width=0.70\textwidth]{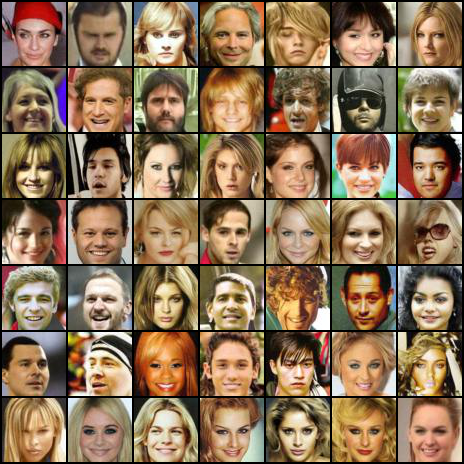}
        \caption{Samples from HR Annealed Langevin\\ on CelebA $64\times 64$ }
        \label{fig:celeba_samples}
    \end{subfigure}    
 \caption{Annealed Langevin  results in high quality  64x64 images (FID 20.17) and does not suffer from the curse of dimension giving evidence of the manifold hypothesis, and the graceful dependency of the mixing property of Langevin sampling on the intrinsic dimension of the image manifold.  }
\label{fig:scratch}
 \vspace{-1.em}
\end{figure}
\begin{figure}
\vspace{-0.2em}
    \begin{subfigure}[L]{0.4\textwidth}
    \centering
        \includegraphics[width=\textwidth]{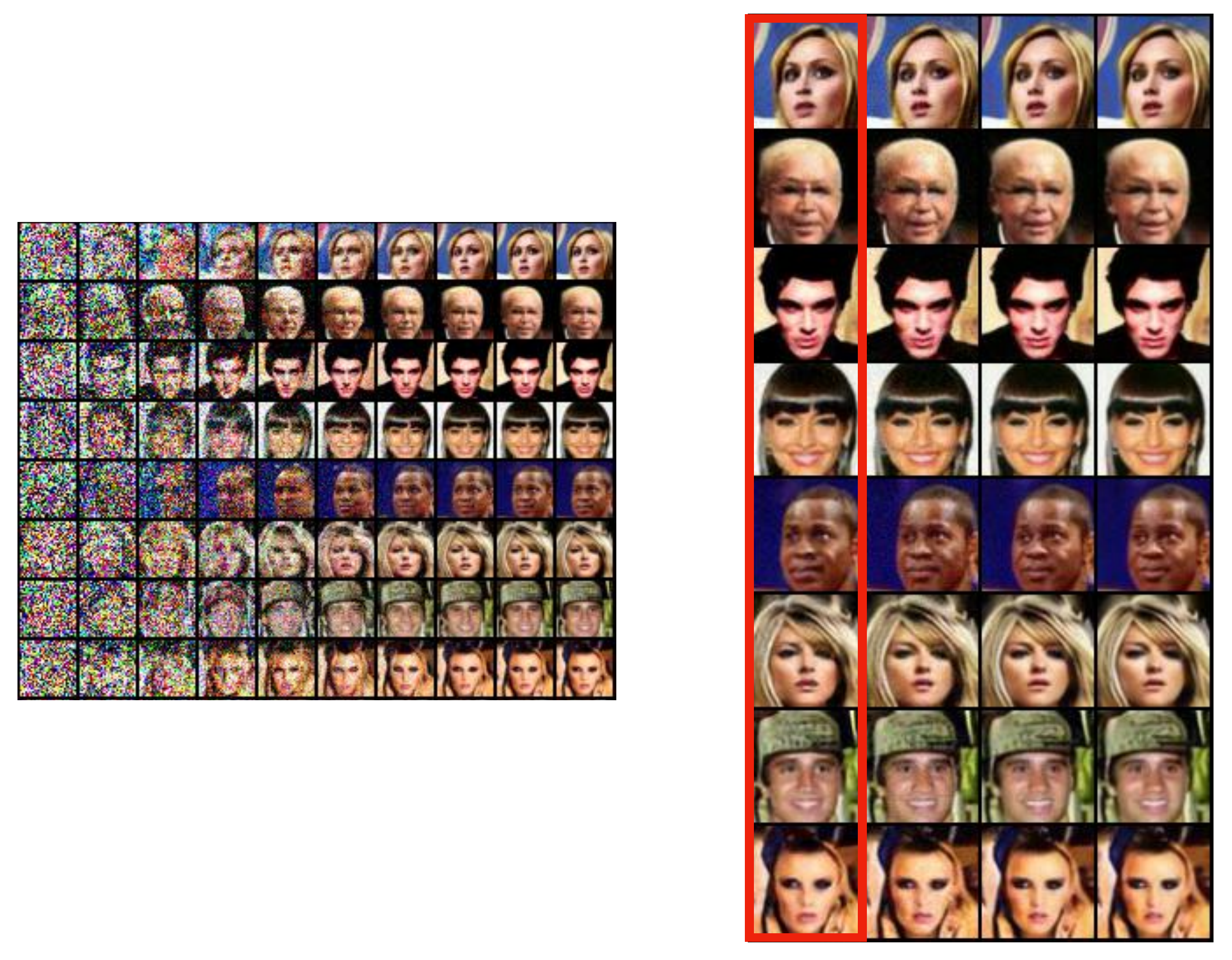}
        \caption{Trajectories of multi-resolution Langevin LRS-$\uparrow$-HRS-3}
        \label{fig:3traj}
    \end{subfigure}
    \begin{subfigure}[R]{0.6\textwidth}
    \centering
        \includegraphics[width=\textwidth]{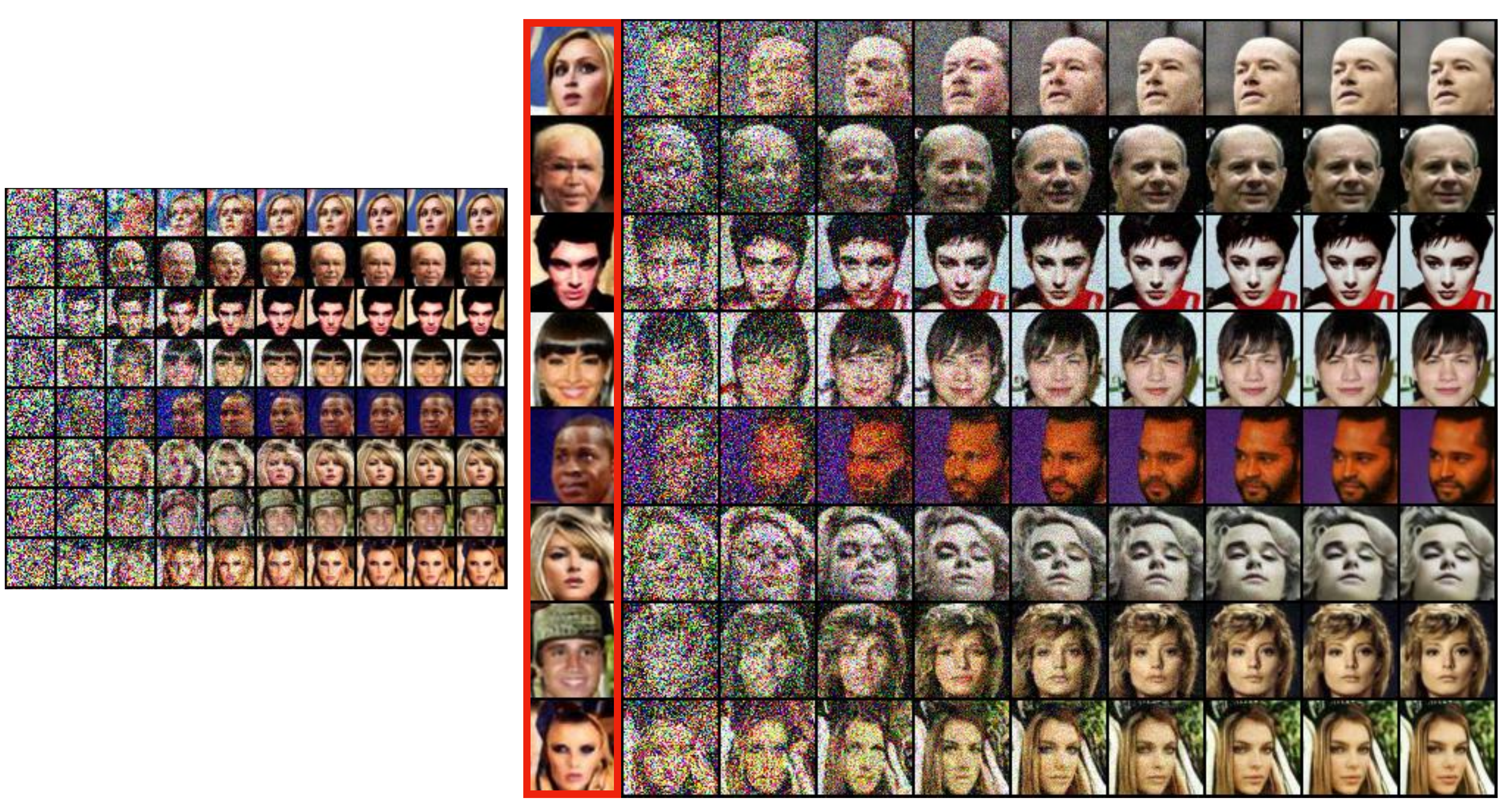}
        \caption{Trajectories of multi-resolution Langevin LRS-$\uparrow$-HRS-9}
        \label{fig:9traj}
    \end{subfigure}
    
    \caption{ Trajectories of Annealed  multi-resolution Langevin sampling: On the left of each panel the low resolution annealed Langevin, each column corresponds to a noise level. On the right of the panel, in red the upsampled image, followed by the high resolution annealed Langevin Sampling. In LRS-$\uparrow$-HRS-3 (Left panel), we use only the three smallest level of noises, that further sharpen the up-sampled images  resulting in the reduction in the FID from 37.2 to 32.46 wrt to LRS-$\uparrow$.In LRS-$\uparrow$-HRS-9 (right panel), we use 9 levels of noise starting from the second largest one, we see that the high resolution Langevin changes the identity of the face while keeping the poses and the facial expressions and results in a  reduction of FID from 37.2 to 26.44.}
    \label{fig:samples}
\vskip -0.25in    
\end{figure}

\section{Discussion}
    The experiments above provide further evidence, beyond that in \cite{SongErmon2019}, that the Langevin Dynamics (LD) mix quickly.  The ability to move from low to high resolution without appreciable effect on image quality provides strong evidence in favor of the intrinsic geometry of the image space being the determining factor in the mixing time.  Moreover, the lack of significant improvement using the mutli-resolution scheme suggests that the dependence on dimension in the final term of \Cref{eq2} in \Cref{lastthm} from \cite{DAESGLD} is loose, motivating further work in this direction. 
    
    Another interesting phenomenon can be observed in \Cref{fig:samples}(b).  The columns progress along the trajectory of annealed Langevin Dynamics, initialized at images upsampled from lower resolution.  It is apparent that general features, such as face angle, are preserved, while smaller details, such as color and background, are changed.  The nearest neighbors analysis of \cite{SongErmon2019} demonstrates that the algorithm is not merely learning the data set as attractors, suggesting that the trajectory of the LD is actually moving along the image manifold.  We leave to future work the job of further exploring this phenomenon and the acceleration of the above algorithm, where LD is run in a latent space representation, perhaps learned in a preprocessing step, more in the flavor of GANs. 

\section*{Broader Impact}
Methods for image generation have a broader impact on our understanding of the structure of natural images and have numerous applications in industry. However, it should be noted that any bias in the dataset is likely to be carried over to the image generation mechanism.

\begin{ack}

 We acknowledge the support from NSF under award DMS-1953181 as well as NSF Graduate Research Fellowship support under Grant No. 1122374. We would also like to acknowledge the support from the MIT-IBM Watson AI Lab.

\end{ack}

\bibliographystyle{plainnat}
\bibliography{BIB}

\begin{thebibliography}{43}
\providecommand{\natexlab}[1]{#1}
\providecommand{\url}[1]{\texttt{#1}}
\expandafter\ifx\csname urlstyle\endcsname\relax
  \providecommand{\doi}[1]{doi: #1}\else
  \providecommand{\doi}{doi: \begingroup \urlstyle{rm}\Url}\fi

\bibitem[Bakry and Emery(1985)]{Emery1985}
D.~Bakry and M.~Emery.
\newblock Diffusions hypercontractives.
\newblock In \emph{Lecture Notes in Mathematics}, pages 177--206. Springer
  Berlin Heidelberg, 1985.
\newblock \doi{10.1007/bfb0075847}.
\newblock URL \url{https://doi.org/10.1007/bfb0075847}.

\bibitem[Bardet et~al.(2018)Bardet, Gozlan, Malrieu, and Zitt]{Bardet2018}
Jean-Baptiste Bardet, Nathaël Gozlan, Florent Malrieu, and Pierre-Andre Zitt.
\newblock Functional inequalities for gaussian convolutions of compactly
  supported measures: Explicit bounds and dimension dependence.
\newblock \emph{Bernoulli}, 24\penalty0 (1):\penalty0 333--353, February 2018.
\newblock \doi{10.3150/16-bej879}.
\newblock URL \url{https://doi.org/10.3150/16-bej879}.

\bibitem[Berthelot et~al.(2020)Berthelot, Milanfar, and
  Goodfellow]{Berthelot2020}
David Berthelot, Peyman Milanfar, and Ian~J. Goodfellow.
\newblock Creating high resolution images with a latent adversarial generator.
\newblock \emph{ArXiv}, abs/2003.02365, 2020.

\bibitem[Block et~al.(2020)Block, Mroueh, and Rakhlin]{DAESGLD}
Adam Block, Youssef Mroueh, and Alexander Rakhlin.
\newblock Generative modeling with denoising auto-encoders and langevin
  sampling, 2020.

\bibitem[Burt and Adelson(1983)]{Burt1983}
P.~Burt and E.~Adelson.
\newblock The laplacian pyramid as a compact image code.
\newblock \emph{{IEEE} Transactions on Communications}, 31\penalty0
  (4):\penalty0 532--540, April 1983.

\bibitem[Chen(2011)]{SGapChen2011}
Mu-Fa Chen.
\newblock General estimate of the first eigenvalue on manifolds.
\newblock \emph{Frontiers of Mathematics in China}, 6\penalty0 (6):\penalty0
  1025--1043, December 2011.
\newblock \doi{10.1007/s11464-011-0164-3}.
\newblock URL \url{https://doi.org/10.1007/s11464-011-0164-3}.

\bibitem[Chen and Wang(1997)]{SGapChen1997}
Mufa Chen and Fengyu Wang.
\newblock General formula for lower bound of the first eigenvalue on riemannian
  manifolds.
\newblock \emph{Science in China Series A: Mathematics}, 40\penalty0
  (4):\penalty0 384--394, April 1997.
\newblock \doi{10.1007/bf02911438}.
\newblock URL \url{https://doi.org/10.1007/bf02911438}.

\bibitem[Denton et~al.(2015)Denton, Chintala, szlam, and Fergus]{Denton2015}
Emily~L Denton, Soumith Chintala, arthur szlam, and Rob Fergus.
\newblock Deep generative image models using a laplacian pyramid of adversarial
  networks.
\newblock In C.~Cortes, N.~D. Lawrence, D.~D. Lee, M.~Sugiyama, and R.~Garnett,
  editors, \emph{Advances in Neural Information Processing Systems 28}, pages
  1486--1494. Curran Associates, Inc., 2015.

\bibitem[Do~Carmo(1992)]{DoCarmo1992}
Manfredo Do~Carmo.
\newblock \emph{Riemannian geometry}.
\newblock Birkhauser, Boston, 1992.
\newblock ISBN 978-0-8176-3490-2.

\bibitem[Dominique~Bakry and Ledoux(2014)]{Bakry2014}
Ivan~Gentil Dominique~Bakry and Michel Ledoux.
\newblock \emph{Analysis and Geometry of Markov Diffusion Operators}.
\newblock Springer, Cham, 2014.
\newblock ISBN 978-3-319-00226-2.

\bibitem[Dong et~al.(2016)Dong, Loy, and Tang]{dong2016accelerating}
Chao Dong, Chen~Change Loy, and Xiaoou Tang.
\newblock Accelerating the super-resolution convolutional neural network.
\newblock In \emph{European conference on computer vision}, pages 391--407.
  Springer, 2016.

\bibitem[Fefferman et~al.(2016)Fefferman, Mitter, and Narayanan]{Fefferman2016}
Charles Fefferman, Sanjoy Mitter, and Hariharan Narayanan.
\newblock Testing the manifold hypothesis.
\newblock \emph{Journal of the American Mathematical Society}, 29, February
  2016.
\newblock \doi{10.1090/jams/852}.
\newblock URL \url{https://doi.org/10.1090/jams/852}.

\bibitem[Goodfellow et~al.(2014)Goodfellow, Pouget-Abadie, Mirza, Xu,
  Warde-Farley, Ozair, Courville, and Bengio]{Goodfellow2014}
Ian~J. Goodfellow, Jean Pouget-Abadie, Mehdi Mirza, Bing Xu, David
  Warde-Farley, Sherjil Ozair, Aaron Courville, and Yoshua Bengio.
\newblock Generative adversarial nets.
\newblock In \emph{Proceedings of the 27th International Conference on Neural
  Information Processing Systems - Volume 2}, NIPS’14, page 2672–2680,
  Cambridge, MA, USA, 2014. MIT Press.

\bibitem[Heusel et~al.(2017{\natexlab{a}})Heusel, Ramsauer, Unterthiner,
  Nessler, and Hochreiter]{FID}
Martin Heusel, Hubert Ramsauer, Thomas Unterthiner, Bernhard Nessler, and Sepp
  Hochreiter.
\newblock Gans trained by a two time-scale update rule converge to a local nash
  equilibrium.
\newblock In \emph{Proceedings of the 31st International Conference on Neural
  Information Processing Systems}, NIPS’17, page 6629–6640, Red Hook, NY,
  USA, 2017{\natexlab{a}}. Curran Associates Inc.
\newblock ISBN 9781510860964.

\bibitem[Heusel et~al.(2017{\natexlab{b}})Heusel, Ramsauer, Unterthiner,
  Nessler, and Hochreiter]{heusel2017gans}
Martin Heusel, Hubert Ramsauer, Thomas Unterthiner, Bernhard Nessler, and Sepp
  Hochreiter.
\newblock Gans trained by a two time-scale update rule converge to a local nash
  equilibrium.
\newblock In \emph{Advances in neural information processing systems}, pages
  6626--6637, 2017{\natexlab{b}}.

\bibitem[Heusel et~al.(2017{\natexlab{c}})Heusel, Ramsauer, Unterthiner,
  Nessler, Klambauer, and Hochreiter]{TTUR}
Martin Heusel, Hubert Ramsauer, Thomas Unterthiner, Bernhard Nessler, Gunter
  Klambauer, and Sepp Hochreiter.
\newblock Gans trained by a two time-scale update rule converge to a nash
  equilibrium.
\newblock \emph{NeurIPS}, 2017{\natexlab{c}}.

\bibitem[Hwang and Lee(2019)]{Hwang2019}
Seungsu Hwang and Sanghun Lee.
\newblock Integral curvature bounds and bounded diameter with
  bakry{\textendash}emery ricci tensor.
\newblock \emph{Differential Geometry and its Applications}, 66:\penalty0
  42--51, October 2019.
\newblock \doi{10.1016/j.difgeo.2019.05.003}.
\newblock URL \url{https://doi.org/10.1016/j.difgeo.2019.05.003}.

\bibitem[Hyvarinen(2005)]{hyvarinen2005}
Aapo Hyvarinen.
\newblock Estimation of non-normalized statistical models by score matching.
\newblock \emph{Journal of Machine Learning Research}, 6\penalty0
  (Apr):\penalty0 695--709, 2005.

\bibitem[Ingraham et~al.(2019)Ingraham, Riesselman, Sander, and
  Marks]{Ingraham2018}
John Ingraham, Adam Riesselman, Chris Sander, and Debora Marks.
\newblock Learning protein structure with a differentiable simulator.
\newblock In \emph{International Conference on Learning Representations}, 2019.
\newblock URL \url{https://openreview.net/forum?id=Byg3y3C9Km}.

\bibitem[Karatzas(1991)]{Karatzas1991}
Ioannis Karatzas.
\newblock \emph{Brownian motion and stochastic calculus}.
\newblock Springer-Verlag, New York, 1991.
\newblock ISBN 978-0-387-97655-6.

\bibitem[Karras et~al.(2018)Karras, Aila, Laine, and
  Lehtinen]{Karras2018progressive}
Tero Karras, Timo Aila, Samuli Laine, and Jaakko Lehtinen.
\newblock Progressive growing of {GAN}s for improved quality, stability, and
  variation.
\newblock In \emph{International Conference on Learning Representations}, 2018.
\newblock URL \url{https://openreview.net/forum?id=Hk99zCeAb}.

\bibitem[Lee et~al.(2018)Lee, Risteski, and Ge]{Ge2018}
Holden Lee, Andrej Risteski, and Rong Ge.
\newblock Beyond log-concavity: Provable guarantees for sampling multi-modal
  distributions using simulated tempering langevin monte carlo.
\newblock In S.~Bengio, H.~Wallach, H.~Larochelle, K.~Grauman, N.~Cesa-Bianchi,
  and R.~Garnett, editors, \emph{Advances in Neural Information Processing
  Systems 31}, pages 7847--7856. Curran Associates, Inc., 2018.

\bibitem[Lee(2018)]{Lee2018}
John Lee.
\newblock \emph{Introduction to Riemannian manifolds}.
\newblock Springer, Cham, Switzerland, 2018.
\newblock ISBN 978-3-319-91754-2.

\bibitem[Li et~al.(2019)Li, Wu, Mackey, and Erdogdu]{Li2019}
Xuechen Li, Denny Wu, Lester Mackey, and Murat~A. Erdogdu.
\newblock Stochastic runge-kutta accelerates langevin monte carlo and beyond.
\newblock In \emph{NeurIPS}, 2019.

\bibitem[{Mallat}(1989)]{Mallat1989}
S.~G. {Mallat}.
\newblock A theory for multiresolution signal decomposition: the wavelet
  representation.
\newblock \emph{IEEE Transactions on Pattern Analysis and Machine
  Intelligence}, 11\penalty0 (7):\penalty0 674--693, 1989.

\bibitem[Moitra and Risteski(2020)]{Risteski2020}
Ankur Moitra and Andrej Risteski.
\newblock Fast convergence for langevin diffusion with matrix manifold
  structure.
\newblock \emph{CoRR}, abs/2002.05576, 2020.
\newblock URL \url{https://arxiv.org/abs/2002.05576}.

\bibitem[Myers(1941)]{Myers1941}
S.~B. Myers.
\newblock Riemannian manifolds with positive mean curvature.
\newblock \emph{Duke Mathematical Journal}, 8\penalty0 (2):\penalty0 401--404,
  June 1941.
\newblock \doi{10.1215/s0012-7094-41-00832-3}.
\newblock URL \url{https://doi.org/10.1215/s0012-7094-41-00832-3}.

\bibitem[Nguyen et~al.(2017)Nguyen, Clune, Bengio, Dosovitskiy, and
  Yosinski]{Nguyen2017}
Anh Nguyen, Jeff Clune, Yoshua Bengio, Alexey Dosovitskiy, and Jason Yosinski.
\newblock Plug {\&} play generative networks: Conditional iterative generation
  of images in latent space.
\newblock In \emph{2017 {IEEE} Conference on Computer Vision and Pattern
  Recognition ({CVPR})}. {IEEE}, July 2017.
\newblock \doi{10.1109/cvpr.2017.374}.
\newblock URL \url{https://doi.org/10.1109/cvpr.2017.374}.

\bibitem[Ohta(2014)]{Ohta2014}
Shin-Ichi Ohta.
\newblock \emph{Ricci curvature, entropy, and optimal transport}, page
  145–200.
\newblock London Mathematical Society Lecture Note Series. Cambridge University
  Press, 2014.

\bibitem[Raginsky et~al.(2017)Raginsky, Rakhlin, and
  Telgarsky]{RakhlinRaginsky}
Maxim Raginsky, Alexander Rakhlin, and Matus Telgarsky.
\newblock Non-convex learning via stochastic gradient langevin dynamics: a
  nonasymptotic analysis.
\newblock In Satyen Kale and Ohad Shamir, editors, \emph{Proceedings of the
  2017 Conference on Learning Theory}, volume~65 of \emph{Proceedings of
  Machine Learning Research}, pages 1674--1703, Amsterdam, Netherlands, 07--10
  Jul 2017. PMLR.
\newblock URL \url{http://proceedings.mlr.press/v65/raginsky17a.html}.

\bibitem[Rose(2019)]{Rose2019}
Christian Rose.
\newblock Almost positive ricci curvature in kato sense -- an extension of
  myers' theorem, 2019.

\bibitem[Rose and Stollmann(2017)]{Rose2017}
Christian Rose and Peter Stollmann.
\newblock The kato class on compact manifolds with integral bounds on the
  negative part of ricci curvature.
\newblock \emph{Proceedings of the American Mathematical Society}, 145\penalty0
  (5):\penalty0 2199--2210, January 2017.
\newblock \doi{10.1090/proc/13399}.
\newblock URL \url{https://doi.org/10.1090/proc/13399}.

\bibitem[Song et~al.(2017)Song, Zhao, and Ermon]{Song2017}
Jiaming Song, Shengjia Zhao, and Stefano Ermon.
\newblock A-nice-mc: Adversarial training for mcmc.
\newblock In \emph{NeurIPS}, 06 2017.

\bibitem[Song and Ermon(2019)]{SongErmon2019}
Yang Song and Stefano Ermon.
\newblock Generative modeling by estimating gradients of the data distribution.
\newblock In H.~Wallach, H.~Larochelle, A.~Beygelzimer, F.~d~Alch\'{e}-Buc,
  E.~Fox, and R.~Garnett, editors, \emph{Advances in Neural Information
  Processing Systems 32}, pages 11918--11930. Curran Associates, Inc., 2019.

\bibitem[Sprouse(2000)]{Sprouse2000}
Chadwick Sprouse.
\newblock Integral curvature bounds and bounded diameter.
\newblock \emph{Communications in Analysis and Geometry}, 8\penalty0
  (3):\penalty0 531--543, 2000.

\bibitem[Sutherland et~al.(2018)Sutherland, Strathmann, Arbel, and
  Gretton]{Gretton2018}
Dougal Sutherland, Heiko Strathmann, Michael Arbel, and Arthur Gretton.
\newblock Efficient and principled score estimation with nyström kernel
  exponential families.
\newblock In Amos Storkey and Fernando Perez-Cruz, editors, \emph{Proceedings
  of the Twenty-First International Conference on Artificial Intelligence and
  Statistics}, volume~84 of \emph{Proceedings of Machine Learning Research},
  pages 652--660, Playa Blanca, Lanzarote, Canary Islands, 09--11 Apr 2018.
  PMLR.
\newblock URL \url{http://proceedings.mlr.press/v84/sutherland18a.html}.

\bibitem[Vincent(2011)]{Vincent2011}
Pascal Vincent.
\newblock A connection between score matching and denoising autoencoders.
\newblock \emph{Neural Computation}, 23\penalty0 (7):\penalty0 1661--1674, July
  2011.
\newblock \doi{10.1162/neco_a_00142}.
\newblock URL \url{https://doi.org/10.1162/neco_a_00142}.

\bibitem[Wang(1997{\natexlab{a}})]{Wang1996}
Feng-Yu Wang.
\newblock Logarithmic sobolev inequalities on noncompact riemannian manifolds.
\newblock \emph{Probability Theory and Related Fields}, 109\penalty0
  (3):\penalty0 417--424, November 1997{\natexlab{a}}.
\newblock \doi{10.1007/s004400050137}.
\newblock URL \url{https://doi.org/10.1007/s004400050137}.

\bibitem[Wang(1997{\natexlab{b}})]{Wang1997}
Feng-Yu Wang.
\newblock On estimation of the logarithmic sobolev constant and gradient
  estimates of heat semigroups.
\newblock \emph{Probability Theory and Related Fields}, 108\penalty0
  (1):\penalty0 87--101, May 1997{\natexlab{b}}.
\newblock \doi{10.1007/s004400050102}.
\newblock URL \url{https://doi.org/10.1007/s004400050102}.

\bibitem[Wang and Wang(2016)]{Wang2016}
Feng-Yu Wang and Jian Wang.
\newblock Functional inequalities for convolution probability measures.
\newblock \emph{Annales de l'Institut Henri Poincar{\'{e}}, Probabilit{\'{e}}s
  et Statistiques}, 52\penalty0 (2):\penalty0 898--914, May 2016.
\newblock \doi{10.1214/14-aihp659}.
\newblock URL \url{https://doi.org/10.1214/14-aihp659}.

\bibitem[Yau(1975)]{SGapYau}
Shing-Tung Yau.
\newblock Isoperimetric constants and the first eigenvalue of a compact
  riemannian manifold.
\newblock \emph{Annales scientifiques de l'\'Ecole Normale Sup\'erieure}, Ser.
  4, 8\penalty0 (4):\penalty0 487--507, 1975.
\newblock \doi{10.24033/asens.1299}.
\newblock URL \url{http://www.numdam.org/item/ASENS_1975_4_8_4_487_0}.

\bibitem[Yeo(2019)]{code}
Jeffrey Yeo, 2019.
\newblock URL \url{https://github.com/yjn870/FSRCNN-pytorch}.

\bibitem[Zhang et~al.(2020)Zhang, Zhang, Odena, and Lee]{Zhang2020}
Han Zhang, Zizhao Zhang, Augustus Odena, and Honglak Lee.
\newblock Consistency regularization for generative adversarial networks.
\newblock In \emph{International Conference on Learning Representations}, 2020.
\newblock URL \url{https://openreview.net/forum?id=S1lxKlSKPH}.

\end{thebibliography}

\newpage
\appendix
\section{Prerequisite Riemannian Geometry}
    We provide a brief review of the relevant concepts from Riemannian geometry.  For an excellent exposition on the topic, see \cite{Lee2018,DoCarmo1992}.
    
    We define an $n$-dimensional \emph{manifold} $M$ as a topological space along with a family $(U_\alpha, \phi_\alpha)$ where  $U_\alpha$ are open sets such that $M = \bigcup_\alpha U_\alpha$ and $\phi_\alpha: \mathbb{R}^n \to U_\alpha$ are bijections such that $\phi_\beta^{-1} \circ \phi_\alpha: \mathbb{R}^n \to \mathbb{R}^n$ are smooth functions.  The \emph{tangent space} of $M$ at a point $x \in M$ is the set of all vectors tangent to $M$ at $x$.  The \emph{tangent bundle} $\mathcal{T}M$ is the set of all pairs $(x, v)$ such that $x \in M$ and $v$ is in the tangent space of $M$ at $x$.  The differential of any smooth map $\phi: M \to N$ between smooth manifolds is a linear map on the tangent bundle sending a vector $v \in \mathcal{T}M_x$ to a vector $v' \in \mathcal{T}N_{f(x)}$.  Crucially, a metric allows us to define a notion of distance on $M$.  Let $I = [0,1]$ be the unit interval and let $\gamma: I \to M$ be a piecewise differentiable map.  Then we note that for any $t \in (0,1)$, $\gamma'(t) \in \mathcal{T}M_{\gamma(t)}$.  Thus the Riemannian metric allows us to define the \emph{length} of $\gamma$ to be
    \begin{equation}
        \ell(g) = \int_0^1 g(\gamma'(t), \gamma'(t)) d t
    \end{equation}
    We can then define a metric $\rho$ on $M$ such that for all $x,y \in M$, $\rho(x,y)$ is the infimum of the length of $\gamma$ taken over the set of curves such that $\gamma(0) = x$ and $\gamma(1)=  y$.  We define the \emph{diameter} of $M$ as the supremum over all $x, y \in M$ of $\rho(x,y)$.  We say that $M$ is \emph{complete} if it is complete with respect to the metric $\rho$.
    
    Given a manifold, we define a \emph{Riemannian metric} as a symmetric 2-tensor that induces an inner product on the tangent space to $M$ at any point $x \in M$.  We call the pair $(M, g)$ a Riemannian manifold.  If $\phi: M \to N$ is a smooth bijection, and $g$ is a Riemannian metric on $M$, then $\phi$ induces a Riemannian metric on $N$ by pushforward through the differential of $\phi$.  If $N$ has a Riemannian metric $g'$ and $\phi g = g'$ then we say that $\phi$ is an \emph{isometry}.  If $\phi$ induces a homeomorphism onto its image then we say that $\phi$ is an \emph{isometric embedding}.
    \begin{remark}
        We care about maps $\phi: M \to N \subset \mathbb{R}^n$ given by inclusion.  There is a natural Riemannian metric on Euclidean space, namely the standard inner product, and this induces a metric on $N$.  The key point in the conditions in Section 4 is that this induced metric agrees with the Riemannian metric of $M$ when considered as an abstract manifold; this is why we require the inclusion to be isometric in the fundamental assumption made at the beginning of this section.
    \end{remark}
    We say that $M$ is \emph{connected} if it is connected as a topological space.  Recall that a metric induces a volume form, whose value at a point $x \in M$ is given by the square root of the absolute value of the determinant of $g$ at this point, i.e., $\vol_M = \sqrt{|g|}$.
    
    We recall that the Levi-Civita connection is the unique symmetric connection compatible with the metric $g$.  With the notion of derivative defined, we are able to extend the classical differetnial operators to the setting of Riemannian manifolds.  In particular, we may define the gradient and the divergence with respect to the Levi-Civita connection.  We then may define the metric Laplacian, or \emph{Laplace-Beltrami operator} as the divergence of the gradient, both with respect to the Levi-Civita connection.  The \emph{Hessian}, denoted by $H_\cdot$ is the covariant derivative iterated twice.  We then define the \emph{Riemannian curvature tensor} as the tensor endomorphism parametrized by $X,Y$ such that
    \begin{equation*}
        R(X,Y)(Z) = \nabla_X \nabla_Y Z - \nabla_Y \nabla_X Z - \nabla_{[X,Y]} Z
    \end{equation*}
    Fixing a frame, there are four coordinates relevant to this map, corresponding to the two vector fields that parametrize the map $R(\cdot, \cdot)$, as well as the input and output of this endomorphism.  We define the \emph{Ricci curvature} as the two tensor taking in two vector fields and returning a real number:
    \begin{equation*}
        \ric(X, Y) = \trace(Z \mapsto R(Z, X)(Y))
    \end{equation*}
    One way of thinking about the Ricci tensor (after fixing a frame) is as a function that takes points $x$ in $M$ to matrices on the tangent space of $M$ at $x$.  Thus, after fixing a point $x \in M$, we can apply the machinery of linear algebra to $\ric(X, Y)$ such as characterizations of eigenvalues and positive definiteness.  If we make such a statement about the Ricci tensor, we are saying that the property stated holds uniformly for all points in $M$.  Finally, we define the scalar curvature as the trace of the Ricci tensor; thus, the scalar curvature is a real valued function on the manifold.
    
    We conclude this brief review with a statement of the famous Bishop-Gromov Comparison Theorem.  This is well known and can be found in any book on comparison geometry; we refer to the notes in \cite{Ohta2014} for an excellent introduction to the basic theory at an elementary level:
    \begin{theorem}[Bishop-Gromov]\label{bishopgromov}
        Let $(M, g)$ a $d'$-dimensional Riemannian manifold such that the scalar curvature is bounded below by $-K$ for some $K > 0$.  For any point $x \in M$, let $B_r(x)$ denote the metric ball around $x$ in $M$ with radius $r > 0$.  For any $0 < r < R$, we have
        \begin{equation}
            \frac{\vol_M(B_R(x))}{\vol_M(B_r(x))} \leq \frac{\int_0^R s(u)^{d'-1} d u}{\int_0^r s(u)^{d'-1} d u}
        \end{equation}
        where 
        \begin{equation}
            s(u) = \sinh(u \sqrt{K})
        \end{equation}
    \end{theorem}
    An immediate corollary of the above is
    \begin{corollary}
        In the situation of \Cref{bishopgromov}, letting $S_r(x) = \partial B_r(x)$ be the metric $r$-sphere and $\vol_M'$ the induced volume form,
        \begin{equation}
            \frac{\vol_M'(S_R(x))}{\vol_M'(S_r(x))} \leq \frac{s(R)^{d'-1}}{s(r)^{d'-1}}
        \end{equation}
    \end{corollary}
    Proofs of both results are available in the third section of \cite{Ohta2014}.
\section{Langevin Diffusion and the log-Sobolev Inequality}
    We briefly review the prerequisite information about both Langevin sampling and the log-Sobolev inequality.  More information and proofs can be found in the excellent exposition of \cite{Bakry2014}.
    
    Recall that, in $\mathbb{R}^d$, we may consider the Langevin diffusion process as the solution to the following stochastic differential equation:
    \begin{equation}\label{eq:sde}
        d X_t = \nabla \log p(X_t) d t + \sqrt{2} d B_t
    \end{equation}
    where $p$ is the density of a probability distribution on $\mathbb{R}^d$.  Letting $\nu_t$ be the law of $X_t$, we note that under relatively weak conditions, $\nu_t \to \nu_\infty = p$ in distribution.  If the drift is Lipschitz, then there is a unique diffusion satisfying \Cref{eq:sde} (see, for example, \cite{Karatzas1991}).  We define the \emph{generator} of this diffusion $\mathcal{L}$ as the second order differential operator:
    \begin{equation}\label{eq3}
        \mathcal{L}f = \Delta f + \langle \nabla \log p, \nabla f \rangle
    \end{equation}
    for all $f \in C_0^2(\mathbb{R})$.  Associated to this generator is the \emph{Dirichlet form} defined as follows:
    \begin{equation}
        \mathcal{E}(f) = -\int f \mathcal{L} f d p = \int ||\nabla f||^2 d p
    \end{equation}
    The \emph{entropy} of a nonnegative function is defined as 
    \begin{equation}
        \ent(f) = \mathbb{E}\left[ f \log f\right] - \mathbb{E}[f] \log \mathbb{E}[f]
    \end{equation}
    Note that if $f$ has mean 1 then the second term drops out.  Our purpose in the above exposition is to introduce the log-Sobolev inequality.  We say that $p$ satisfies a \emph{log-Sobolev inequality} with log-Sobolev constant $c_{LS}$ if for all $f \in C_0^2(\mathbb{R}^d)$,
    \begin{equation}
        \ent(f^2)\leq c_{LS} \mathcal{E}(f)
    \end{equation}
    This seemingly simple inequality is the key to fast mixing in Wasserstein distance for the Langevin diffusion.  We note that if we suppose that $\mathbb{E}[f^2(X)] = 1$, then we may consider the probability measure $\mu \ll p$ such that $\frac{d\mu}{d p} = f^2$.  Then the log-Sobolev inequality becomes equivalent to
    \begin{equation}
        KL(\mu||p) \leq c_{LS} \mathcal{E}\left(\sqrt{\frac{d \mu}{d p}}\right)
    \end{equation}
    for all $\mu \ll p$.  In fact, this seemingly special case implies the general result.
    
    The reason that we care about log-Sobolev inequalities is that they imply fast mixing in Wasserstein distance.  Recalling that we denote by $\nu_t$ the law of $X_t$, the Langevin diffusion at point $t$, we have the following theorem:
     \begin{theorem}
        Let $p$ be a density on $\mathbb{R}^d$ that satisfies a log-Sobolev inequality with constant $c_{LS}$.  Let $\nu_t$ be the law of the Langevin diffusion at time $t$, initialized with law $\nu_0$.  Then for all $t$, we have the following inequality:
        \begin{equation}
            \mathcal{W}_2(\nu_t, p) \leq \mathcal{W}_2(\nu_0, p) e^{-\frac{2 t}{c_{LS}}}
        \end{equation}
    \end{theorem}
    Thus the log-Sobolev constant governs the mixing time for the Langevin diffusion.  
    
    The above discussion was based in Euclidean space, but the same notions carry over when generalized to Riemannian manifolds $(M, g)$.  We refer the reader to \cite[\S 3.2]{Bakry2014} for the details.  In this case, we define the generator of the Langevin diffusion as
    \begin{equation}
        \mathcal{L} f = \Delta f + \langle \nabla f, \nabla \log p \rangle_g
    \end{equation}
    where $\Delta$ is the Laplace-Beltrami operator, the gradient is with respect to the Levi-Civita connection, and the inner product is with respect to the Riemannian metric $g$.  Note that in the case that $(M, g)$ is Euclidean space, this reduces to \Cref{eq3}.  With the generator and the distribution, we may define the Dirichlet form and, consequently, the log-Sobolev inequality entirely analogously.  The famous Bakry-Emery criterion (\cite{Emery1985}) guarantees a dimension independent constant in the case of strictly positive curvature.  One version is as follows:
    \begin{theorem}[Bakry-Emery criterion]
        Let $(M,g)$ be a Riemannian manifold and let $p = p d \vol_M$ be a probability density.  Denote by $H$ the Hessian of $\log p$ with respect to the Levi-Civita connection of $g$.  Suppose that for all $x \in M$, $\ric_M(x) - H(x) \geq \alpha g(x)$ in the sense of quadratic forms for some $\alpha > 0$.  Then $p$ satisfies a log-Sobolev inequality with constant $c_{LS} \leq \frac 1\alpha$.
    \end{theorem}
    Note that in the case that $(M, g)$ is just Euclidean space, the Ricci tensor vanishes and the statement reduces to the density being strictly log-concave.
    
    Finally, we recall the notion of spectral gap.  With the generator of the diffusion defined above, we note that, when applied to a constant function $c$, we have $\mathcal{L}c = 0$.  We may then ask what is the next smallest eigenvalue.  In many cases, this smallest eigenvalue is bounded away from zero.  We define the \emph{spectral gap} of the operator, $\lambda^\ast$ as
    \begin{equation}
        \lambda^* = \inf\left\{\mathcal{E}(f) | \int_M f d p = 0 \text{ and } \int_M f^2 d p = 1  \right\} 
    \end{equation}
    The spectral gap of various operators, especially the Laplace-Beltrami operator, is of enourmous interest in certain fields of geometric analysis.  An exhaustive list of references would be tediously long, but bounds on the spectral gap of the Laplace-Beltrami operator can be found in \cite{SGapChen1997,SGapYau,SGapChen2011}, among many others.  In particular, we have the following result, whose proof can be found in \cite{SGapChen2011}:
    \begin{theorem}\label{sgap}
        Let $(M,g)$ be a $d'$-dimensional, connected, compact manifold such that for all points $x \in M$, $\ric_M(x) \geq - K g$ in the sense of quadratic forms and with diameter $D$.  Let $\lambda^*$ be the spectral gap of the Laplace-Beltrami operator.  Then
        \begin{equation}
            \frac 1{\lambda^*} \leq \frac{D^2}{\pi^2} e^{\frac D2 \sqrt{K(d'-1)}}
        \end{equation}
    \end{theorem}
    We will need this result in order to bound the log-Sobolev constant for the case of uniform distributions on $(M,g)$.
\section{Proofs from Section 3}
    We are now ready to prove the main results from section 3, the bounds on the log-Sobolev constant.  We first have
    \begin{lemma}
        If $p$ has log-Sobolev constant $c$ then $c_{LS}(\ps) \leq 2 \sigma^2 + c$.
    \end{lemma}
    \begin{proof}
        Note that the Bakry-Emery criterion implies that the log-Sobolev constant of $\gs$ is $2 \sigma^2$ (see, for example, \cite{Bakry2014}).  By Proposition 1.1 from \cite{Wang2016}, we know that $c_{LS}(\ps) \leq c_{LS}(\gs) + c_{LS}(p)$.  The result follows.
    \end{proof}
    Thus it suffices to bound the log-Sobolev constant of $p$.  To do this, we apply an estimate of \cite{Wang1997}:
	\begin{theorem}[Theorem 3.3 in \cite{Wang1997}]\label{wang1}
		Let $(M,g)$ be a $d'$-dimensional compact, connected manifold without boundary with dimension $D$. Suppose that $K' > 0$ such that $\ric_M - H_{\log p} \geq - K' g$ in the sense of quadratic forms and let
		\begin{align}
			R_{\nabla \log p} = \sup_{\substack{v \in TM \\ ||v||_g = 1}} \langle \nabla \log p, v \rangle_g^2 + \langle \nabla_v \nabla \log p, v \rangle_g - \ric(v,v)
		\end{align}
		Then,
		\begin{equation}
			c_{LS}(p) \leq \frac{e^{2 K' (d'+1)D^2} - 1}{K'} \left(\frac{d' + 2}{d'}\right)^{d' + 1} e^{1 + (d' + 1) R_{\nabla \log p} D^2}
		\end{equation}
	\end{theorem}
	For the special case of a uniform distribution, while the above result applies, we have a finer estimate in terms of the spectral gap, appearing in \cite{Wang1996}:
	\begin{theorem}[Theorem 1.4 in \cite{Wang1996}] \label{wang2}
		Let $(M, g)$ be a compact Riemannian manifold of dimension $d'$ and diameter $D'$ and let $K'$ be as in \Cref{wang1}.  Let $\lambda^*$ be the spectral gap of the generator of the Langevin diffusion.  Then we have the following bound on the log-Sobolev constant of $p$:
		\begin{equation}
		    c_{LS}(p) \leq \frac 8{\lambda^*} \left(1 + (K^2 + 1)D^2\right) \vee \left( \frac 8{\lambda^*} + 1 \right)
		\end{equation}
	\end{theorem}
	We will see below that $K'$ and $R_{\nabla \log p}$ are easy to bound in the quantities of interest.  In order to get a useful overall bound, we need control of the diameter of the manifold.  We invoke a generalization of the classical diameter bound of Myers \cite{Myers1941} to the regime of negative Ricci curvature.  In order to do this, we recall our definition of the Kato constant from Section 3 as, for any $R > 0$,
	\begin{equation}
	    \kappa(R) = \sup_{x \in M} \frac{1}{\vol(B_R(x))} \int_{B_R(x)} (d' - 1 - \ric_-)_+ d \vol_M
	\end{equation}
	Following the arguments of \cite{Sprouse2000,Hwang2019}, we provide a quantitative diameter bound in terms of the parameters of interest:
	\begin{proposition}\label{diameterbound}
	    Let $(M, g)$ satisfy the assumptions of \Cref{as1} with Kato constant $\kappa = \kappa\left(\sqrt{\frac{d'-1}{K}} \log 4 \right)$.  Assume that $K > 1$.  Then the diameter of $M$, $D$ satisfies
	    \begin{equation}
	        D \leq 8 \sqrt{K(d'-1)} \left(5 + \log\left(\frac{1024 \kappa}{\sqrt{K(d'-1)}} \right) \right) \vee 2 \pi
	    \end{equation}
	\end{proposition}
	\begin{remark}
	    Note that a bound more intrinsic to the parameters of interest in the nonuniform case, in particular relying only on the Bakry-Emery curvature tensor $\ric_M - H_{\log p}$ as opposed to the Ricci tensor and with the Kato constant defined by integrating with respect to $p$ instead of the volume element, can be derived similarly using almost identical arguments.  We adopt the geometric notion for the sake of clarity of presentation and the fact that the dependence does not change very much.  For the more general arguments, see \cite{Hwang2019}.
	\end{remark}
    \begin{proof}
        We follow the argument of \cite{Sprouse2000}.  Let $K = (d'-1) k$ and let $B_k(r)$ be a ball of radius $r$ in the simply connected space of constant sectional curvature $-k$.  We note that in this case, because the Ricci tensor evaluated at a unit vector is the sum of the sectional curvatures with respect to a completed basis for the tangent space at that point, we have that the Ricci curvature for this space is bounded below by $-(d' - 1) k = -K$.  Thus we are in the situation of \Cref{bishopgromov}.  Let $|B_k(r)|$ and $|\partial B_k(r)|$ denote the volume of said ball and the boundary of said ball respectively.  Now, following the proof of \cite[Theorem 1.4]{Sprouse2000}, we note that if $\delta > 2 \pi$ and
        \begin{equation}
            \kappa(R) \leq \epsilon(\delta) := \frac{2}{\delta} \frac{|\partial B_k(R/2)|}{|\partial B_k(R)|} \frac{B_k(\delta/4)|}{|B_k(R)| + |B_k(2R)|} \frac{(d'-1)(\pi + \delta/2)}{4}
        \end{equation}
        then the diameter of $M$ is bounded above by $\delta$.  If we find some $\epsilon'(\delta)$ such that $\epsilon'(\delta) \leq \epsilon(\delta)$ for all $\delta$ and we ensure that $\delta$ is sufficiently large so as to guarantee that $\kappa \leq \epsilon'(\delta)$, then we know that the diameter is bounded by $\delta$.  Note first that $|B_k(R)| \leq |B_k(2 R)|$.  Thus we may apply \Cref{bishopgromov} to the volume ratios of both the balls and the boundaries of the balls to get
        \begin{align}
            \epsilon(\delta) \geq \frac{d'-1}{8} \frac{s(R/2)^{d'-1}}{s(R)^{d'-1}} \frac{\int_0^{\frac \delta 4} s(t)^{d'-1} d t }{\int_0^{2R} s(t)^{d'-1} d t}
        \end{align}
        where we recall that $s(t) = \sinh\left( \sqrt{\frac{K}{d'-1}}t \right)$.  Now, we note that $s(x) \leq \frac 12 e^{\sqrt{\frac K{d'-1}} x}$ and for any $x \geq \frac 12 \sqrt{\frac{d'-1}{K}} \log 2$, we have $s(x) \geq \frac 14 e^{\sqrt{\frac{K}{d'-1}} x}$.  Now, set $R = \sqrt{\frac{d'-1}{K}} \log 2$ and suppose that $\delta \geq 4 \log (2) \sqrt{\frac{d' - 1}{K}}$.  Then we may bound the above expression by
        \begin{align}
            \epsilon(\delta) &\geq \frac{d' - 1}{8} \left(\frac{\frac 14 e^{\sqrt{\frac{K}{d'-1}} \frac R2 }}{\frac 12 e^{\sqrt{\frac{K}{d'-1}} R}}\right)^{d' - 1} \frac{\frac \delta 8 \left(\frac 14 e^{\sqrt{\frac{K}{d'-1}} \frac{\delta}{8}} \right)^{d'-1}}{2 R \left(\frac 12 e^{\sqrt{\frac{K}{d'-1}} 2 R} \right)^{d'-1}} \\
            &= \frac{d'-1}{128 R}4^{1-d'} \delta \exp\left(\sqrt{(d'-1)K} \left(\frac \delta 8 - \frac 52 R  \right)  \right) = \epsilon'(\delta)
        \end{align}
        because $s$ is an increasing function.   Thus, in order to bound the diameter, we need to find a $\delta$ such that $\kappa \leq \epsilon'(\delta)$.  Note that if we set
        \begin{equation}
            \delta = 8 \sqrt{K(d'-1)} \left(5 + \log\left(\frac{1024 \kappa}{\sqrt{K(d'-1)}} \right) \right)
        \end{equation}
        then we get that $\kappa \leq \epsilon'(\delta) \leq \epsilon(\delta)$ and so the argument in \cite{Sprouse2000} implies that the diameter is bounded by $\delta$.  Note that this value of $\delta$ is automatically greater than $4 \log(2) \sqrt{\frac{d'-1}{K}}$ and thus we are in no trouble from our above assumption.
    \end{proof}
    With the diameter bound proven, we are now ready to prove the main results.  We restate and prove the two main results in Section 3:
    \begin{theorem}
        Suppose that $(M,g)$ satisfies \Cref{as1}.  Let $p = p \vol_M$ be a probability measure that is absolutely continuous with respect to the volume form and suppose that $\nabla \log p$ is $L$-Lipschitz and that $||\nabla \log p||_g \leq B$ at all points.  Let $\kappa = \kappa\left(\sqrt{\frac{d'-1}{K}} \log 2 \right)$  Then, if $K > \frac 1{d'}$,
        \begin{equation}
            c_{LS}(p_{\sigma^2}) = \widetilde{O}\left(\sigma^2 + d'^2 K\log \kappa e^{M B^2 d'^2 \log^2 \kappa} \right)
        \end{equation}
        where $\widetilde{O}$ indicates that we are ignoring factors logarithmic in $d'$ and $K$.
    \end{theorem}
    \begin{proof}
        Recalling from earlier that $c_{LS}(p_{\sigma^2}) \leq 2 \sigma^2 + c_{LS}(p)$, we note that it suffices to bound the log-Sobolev constant of $p$.
        
        With \Cref{wang1,diameterbound}, there is the simple manner of plugging in the constants.  By \Cref{wang1}, we know that
		\begin{equation}
			c_{LS}(p) \leq \frac{e^{2 K' (d'+1)D^2} - 1}{K'} \left(\frac{d' + 2}{d'}\right)^{d' + 1} e^{1 + (d' + 1) R_{\nabla \log p} D^2}
		\end{equation}
		where
		\begin{align}
			R_{\nabla \log p} = \sup_{\substack{v \in TM \\ ||v||_g = 1}} \langle \nabla \log p, v \rangle_g^2 + \langle \nabla_v \nabla \log p, v \rangle_g - \ric(v,v)
		\end{align}
		and $\ric - H_{\log p} \geq - K' g$ at all points, where $H$ is the Hessian.  Similarly, we note that $|\langle \nabla_v \nabla \log p, v \rangle_g| \leq M$ if $v$ is a unit vector (with respect to $g$) in the tangent space.  Similarly, $\langle \nabla \log p, v \rangle_g^2 \leq B^2$.  Putting this together gives that $R_{\nabla \log p} \leq K + M + B^2$.
		
        Now, noting that $\frac 1x (e^x - 1) \leq e^x$ for all $x \geq 0$, we see that
        \begin{equation}
            \frac{e^{2 K' (d'+1)D^2} - 1}{K'} \leq 2(d' + 1) D^2 e^{2 K'(d'+1) D^2}
        \end{equation}
        We also observe that $\left(1 + \frac 2{d'} \right)^{d'} \leq e^2$.  Now, we note that if $\nabla \log p$ is $L$-Lipschitz, then $K' \leq K + M$ as $- M g \leq H_{\log p} \leq M g$.  Putting this together, we get that
        \begin{equation}
            c_{LS}(p) \leq 2(d' + 1) D^2 e^{4 + (d'+1) D^2 (2 K + 2 M + B^2)}
        \end{equation}
        Note that by \Cref{diameterbound}, we have
        \begin{equation}
            D^2 \leq 128 d' K \left(25 + \log \left(\frac{1024 \kappa}{\sqrt{K(d'-1)}} \right)^2 \right)
        \end{equation}
        Plugging this in yields the result, after noting that $K d' \geq 1$ by the assumption on $K$.
    \end{proof}
     While in the uniform case, we can certainly apply the above result, the tight bounds on the spectral gap of the Laplace-Beltrami operator, which doubles as the generator for Langevin dynamics with respect to the uniform measure on a Riemannian manifold, allow us to establish a sharper bound.  Thus, in this special case, we restate and prove:
    \begin{theorem}
         Suppose that the pair $(M, g)$ satisfies \Cref{as1} and let $p \propto \vol_M$ be uniform on $M$.  Assume that $K > 1$ and that $\kappa > 1$.  Then
        \begin{equation}
            c_{LS}(p_{\sigma^2}) = O\left(\sigma^2 + K^4 d'^2 \kappa^{20 K^2 d'}\right)
        \end{equation}
    \end{theorem}
    \begin{proof}
        Using \Cref{wang2} with the spectral gap bound in \Cref{sgap} and the diameter bound in \Cref{diameterbound} immediately yields the result.
    \end{proof}

\section{Miscellaneous Proofs}\label{app1}
    \begin{proof}[Proof of Proposition 1]
        For each $x \in \mathbb{R}^d$, consider the probability distribution on $M$, $q_x$, given by
		\begin{align}
			\frac{d q_x}{d p} = \frac{\gs(x - y)}{\ps(x)}
		\end{align}
		By definition of $\ps$, this integrates to 1.  Now, let
		\begin{align}
			e(x) = \mathbb{E}_{Y \sim q_x} \left[Y\right] && v(x) = \mathbb{E}_{Y \sim q_x}\left[Y Y^t\right] && \Sigma_x = v(x) - e(x) e(x)^t
		\end{align}
		be the mean, second moment, and covariance of $q_x$.  Then we note that
		\begin{align}
			\nabla \log \ps(x) &= \frac{\nabla \ps(x)}{\ps(x)} = \frac{\nabla \int_M g(x - y) d p(y)}{\ps(x)} = \frac 1{\sigma^2 \ps(x)} \int_M ( y -x) g(x - y) d p(y) \\
			&= \frac{1}{\sigma^2} (e(x) - x)
		\end{align}
		Now, note that the Jacobian of $e$ satisfies
		\begin{align}
			J_e(x) &= J\left(\frac{\int_M y \gs(y-x) d p(y)}{\ps(x)} \right) \\
			&= \frac{\int_M y (y - x)^t \gs(y-x) d p(y)}{\sigma^2\ps(x)} - \frac{\int_M y \gs(y-x) d p(y)}{\ps(x)} \frac{\nabla^t \ps(x)}{\ps(x)^2} \\
			&= \frac{v(x) - e(x) x^t }{\sigma^2} - e(x) \nabla^t \log \ps(x) = \frac 1{\sigma^2} \left( v(x) - e(x) \left(x - \sigma^2 \nabla \log \ps(x) \right)^t\right) \\
			&= \frac{1}{\sigma^2} \left(v(x) - e(x) e(x)^t\right) = \frac 1{\sigma^2} \Sigma_x
		\end{align}
		Thus the Hessian of $\log \ps$ is given by
		\begin{align}
			H_{\log \ps}(x) = \frac 1{\sigma^4} \left(\Sigma_x - \sigma^2 I  \right) \preceq \frac{\rho^2 - \sigma^2}{\sigma^4} I
		\end{align}
		Thus the first claim follows.
		
		For the second claim, note that
		\begin{align}
			\inprod{\nabla \log \ps(x)}{x} &= \frac 1{\sigma^2}\inprod{e(x) - x}{x} = \frac{\inprod{e(x)}{x}}{\sigma^2} - \frac 1{\sigma^2} ||x||^2 \\
			&\leq \frac{\rho ||x||}{\sigma^2} - \frac 1{\sigma^2} ||x||^2 \leq - \frac 1{2 \sigma^2} ||x||^2 + \frac{\rho}{2 \sigma^2}
		\end{align}
		where we used the fact that $||e(x)|| \leq \rho$.  This proves the claim.
    \end{proof}
    \begin{proof}[Proof of Proposition 2]
        Let $g$ be a nonnegative, real-valued function in the domain of the Dirichlet form of $p'$, $\mathcal{E}_{p'}$, and suppose that $\mathbb{E}_{p'}[g] = 1$.  Then by the change of variables formula, we have
        \begin{equation}\label{appeq1}
            \ent_{p'}(g) = \int g \log g d p' = \int g(P(x)) \log g(P(x)) d p(x) \leq c_{LS} \mathcal{E}_p(\sqrt{g \circ P})
        \end{equation}
        where the inequality follows from the log-Sobolev inequality assumption on $p$ and we denote by $\mathcal{E}_p$ the Dirichlet form of $p$.  Now, we note that
        \begin{align}\label{appeq2}
            \mathcal{E}_p(\sqrt{g \circ P}) &= \int \frac{\left|\left| J_{P}(x) \nabla g(P(x))\right|\right|^2}{2 g \circ P(x) } d p(x) \leq \int \frac{\left|\left| \nabla g(P(x))\right|\right|^2}{2 g \circ P(x) } d p(x) \\
            &= \int \frac{|| \nabla g||^2}{2 g} d p'(x) = \mathcal{E}_{p'}(\sqrt{g})
        \end{align}
        where the inequality follows from the assumption on the Jacobean of $P$.  Combining \Cref{appeq1,appeq2} yields the result.
    \end{proof}
   
\section{Additional Experiments}\label{app:Plots}

\subsection{Ablation study }
We give here an ablation study corresponding to different multi-resolution schemes in Langevin sampling. We see from these ablations studies that multi-resolutions scheme improves marginally if any on the HRS with Langevin, confirming therefore the manifold hypothesis. 

\begin{table}[ht!]
\begin{center}
\begin{tabular}{lc}
        \toprule
        Method & FID\\
        \midrule
        \multicolumn{2}{l}{\textbf{CELEBA- $64\times 64$ }} \\
        
        \midrule
        HR Langevin (HRS) ~ & $20.17$  \\
        LR Langevin + Up (LRS-$\uparrow$) & $37.20$ \\
        mr-Langevin (LRS-$\uparrow$-HRS-2)&$32.54$ \\
        mr-Langevin (LRS-$\uparrow$-HRS-3)&$32.46$ \\
        mr-Langevin (LRS-$\uparrow$-HRS-4)&$33.39 $ \\
        mr-Langevin (LRS-$\uparrow$-HRS-5)&$34.86 $ \\
        mr-Langevin (LRS-$\uparrow$-HRS-6)&${36.72}$\\
        mr-Langevin (LRS-$\uparrow$-HRS-7)&${36.87}$\\
        mr-Langevin (LRS-$\uparrow$-HRS-8)&${30.78}$\\
        \bottomrule
        
    \end{tabular} 
\end{center}
\caption{ \textbf{Ablation 1 change the starting noise level of HR Langevin.} FID scores for using LRS for all 10 levels followed by upsampling using the upsampling network and $x$ noise levels of HRS sampling Langevin Sampling (LRS-$\uparrow$-HRS-$x$), for e.g for $x=3$ this corresponds to the three last small noises.} \label{tab:scoreablation2}
\vskip -0.23in
\end{table}

\begin{table}[ht!]
\begin{center}
\begin{tabular}{lc}
        \toprule
        Method & FID\\
        \midrule
        \multicolumn{2}{l}{\textbf{CELEBA- $64\times 64$ }} \\
        
        \midrule
        HR Langevin (HRS) ~ & $20.17$  \\
        LR Langevin + Up (LRS-$\uparrow$) & $37.20$ \\
        mr-Langevin (LRS-2$-\uparrow$-HRS-9)&$ \mathbf{19.54} $ \\
        mr-Langevin (LRS-3-$\uparrow$-HRS-8)&$37.77 $ \\
        mr-Langevin (LRS-4-$\uparrow$-HRS-7)&$57.83$ \\
        mr-Langevin (LRS-5-$\uparrow$-HRS-6)&$57.54  $ \\
        mr-Langevin (LRS-6-$\uparrow$-HRS-5)&${43.82 }$\\
        mr-Langevin (LRS-7-$\uparrow$-HRS-4)&${ 35.84}$\\
        mr-Langevin (LRS-8-$\uparrow$-HRS-3)&${32.95 }$\\
        \bottomrule
        
    \end{tabular} 
\end{center}
\caption{\textbf{Ablation 2 change the  last noise level in LR Langevin and the starting noise level of HR Langevin.} FID scores for using LRS for all x levels followed by upsampling using the super-resolution network  and $10-x+1$ noise levels of HRS sampling Langevin Sampling ( for e.g LRS-2$-\uparrow$-HRS-$9$), corresponding to running the first two largest noise levels in LR Langevin followed by upsampling and then run HR Langevin sampling starting from the second noise level (total of 9 noise levels in HR)} \label{tab:scoreablation}
\vskip -0.23in
\end{table}

\begin{table}[ht!]
\begin{center}
\begin{tabular}{lc}
        \toprule
        Method & FID\\
        \midrule
        \multicolumn{2}{l}{\textbf{CELEBA- $64\times 64$ }} \\
        
        \midrule
        HR Langevin (HRS) ~ & $20.17$  \\
        LR Langevin + Up (LRS-$\uparrow$) & $37.20$ \\
        mr-Langevin (LRS-2-$\uparrow$-HRS-9)&$\textbf{19.34}  $ \\
        mr-Langevin (LRS-3-$\uparrow$-HRS-8)&$30.50 $ \\
        mr-Langevin (LRS-4-$\uparrow$-HRS-7)&$38.99 $ \\
        mr-Langevin (LRS-5-$\uparrow$-HRS-6)&$38.21  $ \\
        mr-Langevin (LRS-6-$\uparrow$-HRS-5)&${35.85 }$\\
        mr-Langevin (LRS-7-$\uparrow$-HRS-4)&${ 37.98}$\\
        mr-Langevin (LRS-8-$\uparrow$-HRS-3)&${37.08 }$\\
        \bottomrule
        
    \end{tabular} 
\end{center}
\caption{\textbf{Ablation 3 change the upsampling method, using bi-cubic interpolation.} FID scores for using LRS for all x levels followed by upsampling using the super-resolution network  and $10-x+1$ noise levels of HRS sampling Langevin Sampling ( for e.g LRS-2$-\uparrow$-HRS-$9$), corresponding to running the first two largest noise levels in LR Langevin followed by upsampling and then run HR Langevin sampling starting from the second noise level (total of 9 noise levels in HR)} \label{tab:scoreablation3}
\end{table}

\subsection{Additional Plots}

\begin{figure}[ht!]
    \centering
 \includegraphics[width=\textwidth]{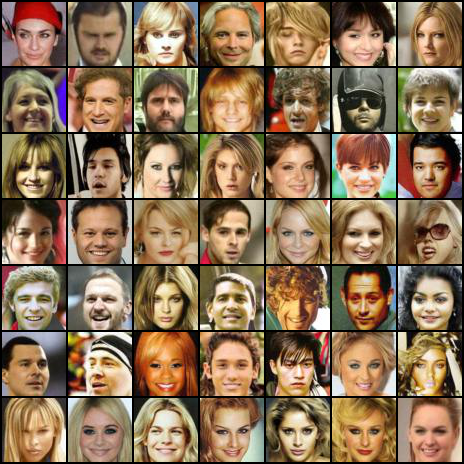}
    \caption{Random samples from HRS Langevin }
    \label{fig:my_label1}
\end{figure}
\begin{figure}[ht!]
    \centering
 \includegraphics[width=\textwidth]{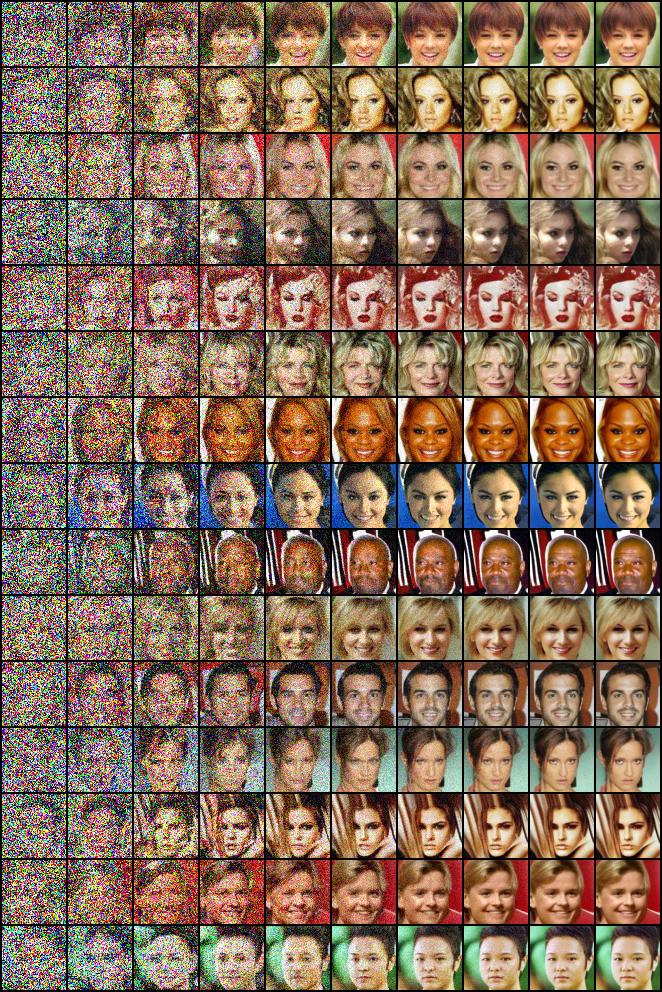}
    \caption{Trajectories  of  HRS Langevin }
    \label{fig:my_label2}
\end{figure}

\begin{figure}[ht!]
    \centering
 \includegraphics[width=\textwidth]{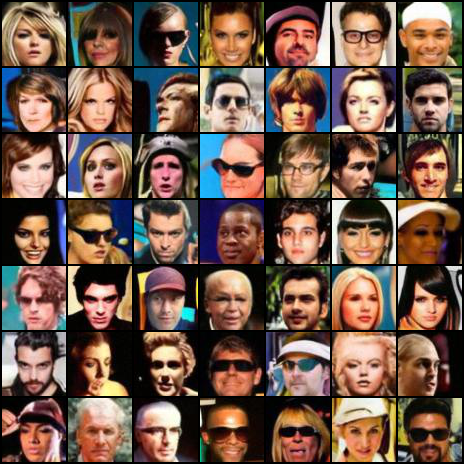}
    \caption{Random Samples from LRS-$\uparrow$-HRS-3 Langevin}
    \label{fig:my_label3}
\end{figure}

\begin{figure}[ht!]
    \centering
 \includegraphics[width=0.4\textwidth]{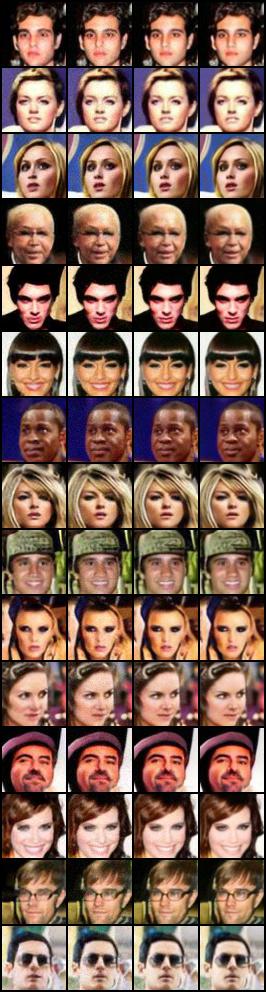}
    \caption{Trajectories of  LRS-$\uparrow$-HRS-3 Langevin, starting from the upsampling of 32x32 Low resolution annealed Langevin( first column is the output of the upsampling ) }
    \label{fig:my_label4}
\end{figure}

\begin{figure}[ht!]
    \centering
 \includegraphics[width=\textwidth]{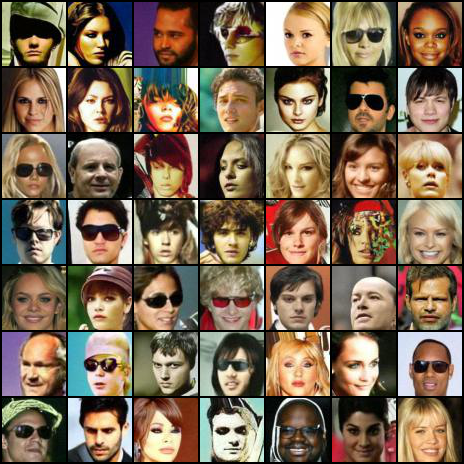}
    \caption{Random Samples from LRS-$\uparrow$-HRS-9 }
    \label{fig:my_label5}
\end{figure}
\begin{figure}[ht!]
    \centering
 \includegraphics[width=\textwidth]{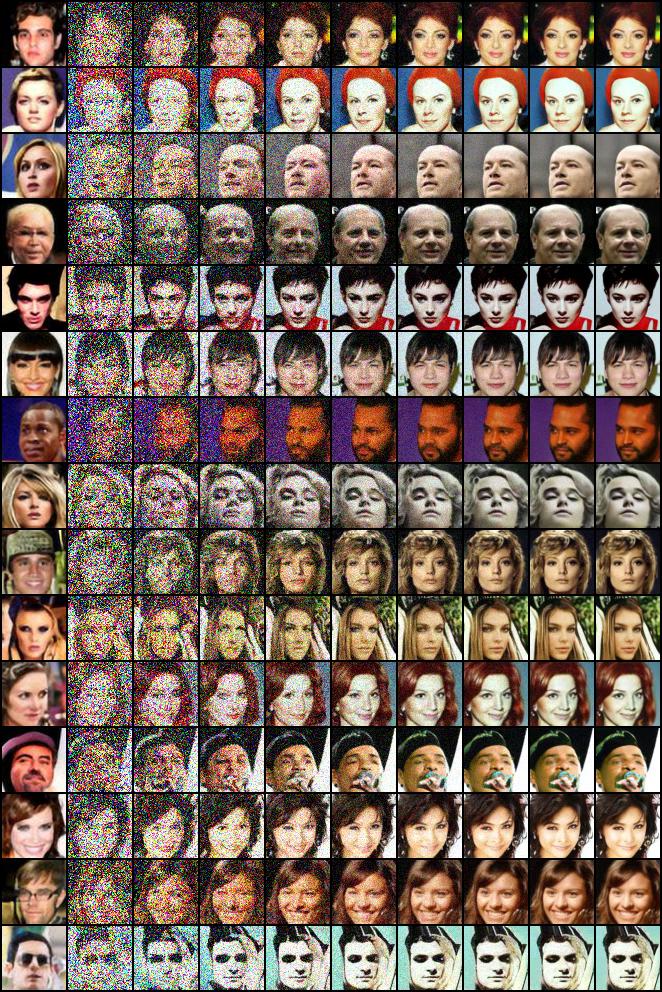}
    \caption{Trajectories of  LRS-$\uparrow$-HRS-9 Langevin, starting from the upsampling of 32x32 Low resolution annealed Langevin ( first column is the output of the upsampling )}
    \label{fig:my_label6}
\end{figure}

\begin{figure}[ht!]
    \centering
 \includegraphics[width=\textwidth]{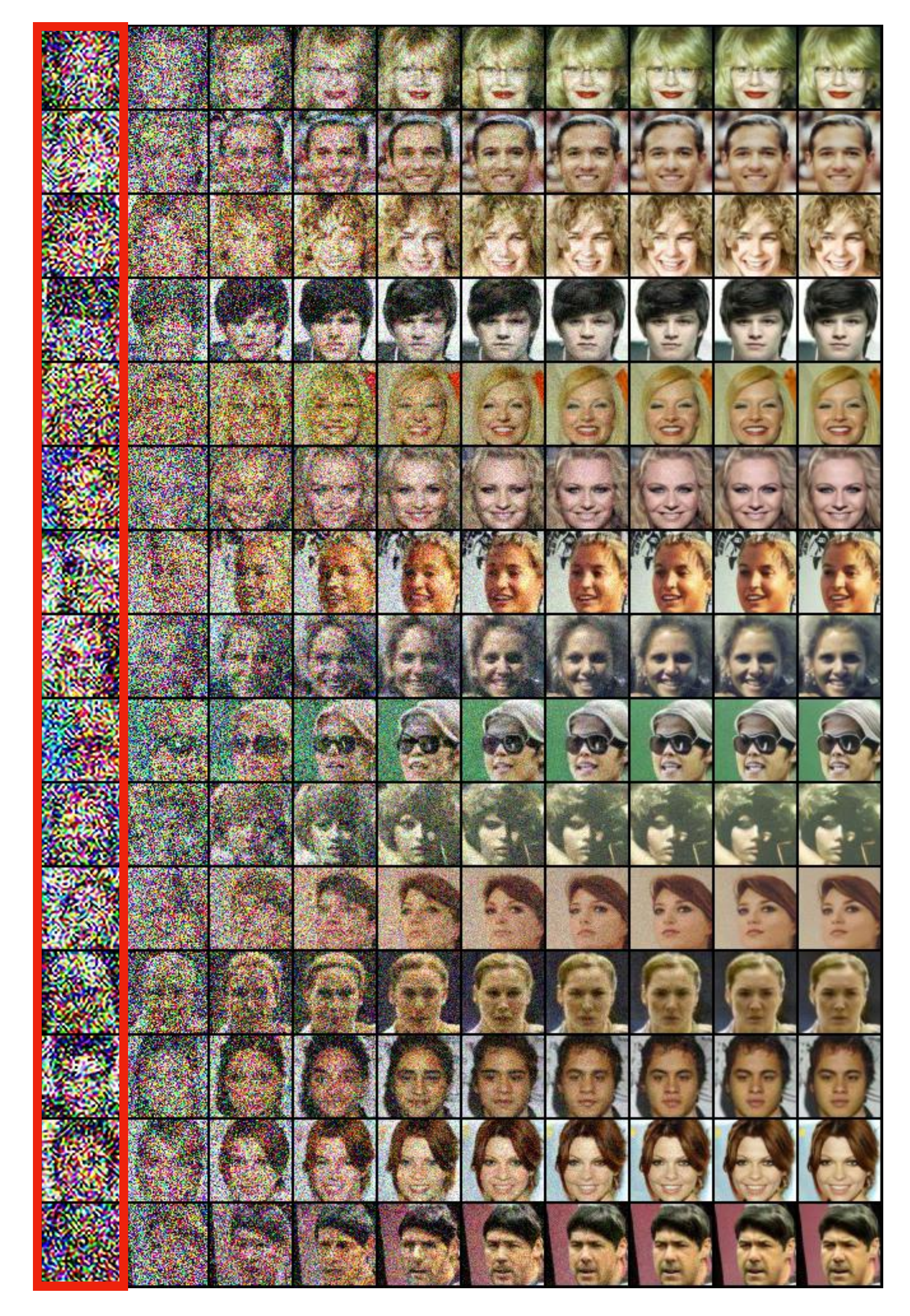}
    \caption{Trajectories of  LRS-2-$\uparrow$-HRS-9 Langevin, starting from the upsampling of 32x32 2 levels of Low resolution annealed Langevin  ( first column in red is the output of the upsampling )}
    \label{fig:my_label9}
\end{figure}

\end{document}